\documentclass[11pt]{article} % For LaTeX2e
\usepackage{fullpage}
\usepackage{graphicx}
\usepackage{algorithm,algorithmic}
\usepackage{amsmath, amssymb, amsthm, bm}
\usepackage{url}
\usepackage{hyperref}
\usepackage{wrapfig}
\usepackage{algorithmic,algorithm}
\usepackage{subfigure}
\newtheorem{theorem}{Theorem}

\newtheorem{lemma}{Lemma}

\newtheorem{definition}{Definition}

        {\hspace*{\fill}$\Box$\par}

\newcommand{\ignore}[1]{}

\newcommand{\R}{\mathbb R}

\newcommand{\z}{\bm{z}}
\newcommand{\eps}{\varepsilon}

\newcommand{\bb}{\bm{b}}

\newcommand{\bg}{\bm{g}}

\newcommand{\bt}{\bm{t}}

\renewcommand{\bt}{\bm{\theta}}
\newcommand{\bto}{\bt^*}
\newcommand{\btt}{\bt^t}
\newcommand{\ttn}{\bt^{t+1}}
\newcommand{\so}{s^*}
\newcommand{\s}{s}
\newcommand{\sa}{q}

\newcommand{\gt}{\bg^t}

\renewcommand{\bb}{\bm{\beta}}
\newcommand{\bbt}{\bb^t}
\newcommand{\bwt}{\widehat{\bt}^t}
\newcommand{\W}{W}

\newcommand{\Wo}{\W^*}

\newcommand{\Sec}[1]{\hyperref[sec:#1]{\S\ref*{sec:#1}}} %section
\newcommand{\Eqn}[1]{\hyperref[eq:#1]{(\ref*{eq:#1})}} %equation
\newcommand{\Fig}[1]{\hyperref[fig:#1]{Fig.\,\ref*{fig:#1}}} %figure
\newcommand{\Tab}[1]{\hyperref[tab:#1]{Tab.\,\ref*{tab:#1}}} %table
\newcommand{\Thm}[1]{\hyperref[thm:#1]{Theorem\,\ref*{thm:#1}}} %theorem
\newcommand{\Lem}[1]{\hyperref[lem:#1]{Lemma\,\ref*{lem:#1}}} %lemma
\newcommand{\Prop}[1]{\hyperref[prop:#1]{Prop.~\ref*{prop:#1}}} %property
\newcommand{\Cor}[1]{\hyperref[cor:#1]{Corollary~\ref*{cor:#1}}} %corollary
\newcommand{\Def}[1]{\hyperref[def:#1]{Definition~\ref*{def:#1}}} %definition
\newcommand{\Alg}[1]{\hyperref[alg:#1]{Alg.~\ref*{alg:#1}}} %algorithm
\newcommand{\Ex}[1]{\hyperref[ex:#1]{Ex.~\ref*{ex:#1}}} %example
\newcommand{\Clm}[1]{\hyperref[clm:#1]{Claim~\ref*{clm:#1}}} %example
\newcommand{\ip}[2]{\langle #1, #2 \rangle}

\title{On Iterative Hard Thresholding Methods for High-dimensional M-Estimation}

\author{
Prateek Jain$^\ast$ \qquad\qquad Ambuj Tewari$^\dagger$ \qquad\qquad Purushottam Kar$^\ast$\\
$^\ast$Microsoft Research, INDIA\\
$^\dagger$University of Michigan, Ann Arbor, USA\\
\texttt{\{prajain,t-purkar\}@microsoft.com, tewaria@umich.edu}
}

\renewcommand{\ss}{L}
\renewcommand{\sc}{\alpha}
\renewcommand{\S}{S}

\newcommand{\So}{\S^*}

\newcommand{\St}{\S^t}

\newcommand{\Stn}{\S^{t+1}}

\newcommand{\It}{I^t}
\newcommand{\vt}{\bm{v}^t}

\begin{document}

\maketitle

\begin{abstract}
The use of M-estimators in generalized linear regression models in high dimensional settings requires risk minimization with hard $L_0$ constraints. Of the known methods, the class of projected gradient descent (also known as iterative hard thresholding (IHT)) methods is known to offer the fastest and most scalable solutions. However, the current state-of-the-art is only able to analyze these methods in extremely restrictive settings which do not hold in high dimensional statistical models. In this work we bridge this gap by providing the first analysis for IHT-style methods in the high dimensional statistical setting. Our bounds are tight and match known minimax lower bounds. Our results rely on a general analysis framework that enables us to analyze several popular hard thresholding style algorithms (such as HTP, CoSaMP, SP) in the high dimensional regression setting. We also extend our analysis to a large family of ``fully corrective methods'' that includes two-stage and partial hard-thresholding algorithms. We show that our results hold for the problem of sparse regression, as well as low-rank matrix recovery.

\end{abstract}
% we show that wolfe's algorithm converges to epsilon approximate solution in $1/\epsilon^3$ time, which is first such analysis for wolfe's algorithm. We then show that by application of our analysis to submodular minimization, we can obtain $O(F^2n^9)$ time algorithm for wolfe's algorithm, which is also first pseudo polytime analysis for wolfe's algorithm applied to submodular minimization. 
%\end{abstract}
\section{Introduction}
Modern statistical estimation is routinely faced with real world problems where the number of parameters $p$ handily outnumbers the number of observations $n$. In general, consistent estimation of parameters is not possible in such a situation. Consequently, a rich line of work has focused on models that satisfy special structural assumptions such as sparsity or low-rank structure. Under these assumptions, several works (for example, see \cite{BuhlmannV2011, NegahbanW2011, RaskuttiWY2011, RohdeT2011,  NegahbanRWY2012}) have established that consistent estimation is information theoretically possible in the ``$n \ll p$" regime as well.

%Several real-world problems routinely generate high-dimensional estimation problems where the number $p$ of parameters is significantly larger than the number $n$ of observations. In general, consistent estimation of parameters in such a situation is not possible. However, several results have shown that under certain special structural assumptions about the true model (e.g., sparsity or low-rank structure), consistent estimation is possible in such a regime (``$n \ll p$ regime") as well (see, e.g., \cite{wainwrightl1, wainwrightmatrix, NegahbanRWY2012, BuhlmannV2011}).

The question of efficient estimation, however, is faced with feasibility issues since consistent estimation routines often end-up solving NP-hard problems. Examples include sparse regression which requires loss minimization with sparsity constraints and low-rank regression which requires dealing with rank constraints which are not efficiently solvable in general \cite{Natarajan1995}.

%Given the \emph{possibility} of estimation, one then turns toward the problem of \emph{efficient} estimation. As it turns out, the estimation routines that are guaranteed to deliver a consistent estimate often end up requiring solutions to NP-hard problems. For instance, in sparse regression, one is required to solve an $L_0$-constrained loss minimization problem which is not efficiently solvable in general \cite{Natarajan1995}. Similarly in the case of low-rank regression, the solution to a rank-constrained minimization problem is required which is again NP-hard.

%Imposing structure on the model guarantees consistent estimation. However, the obtained optimization problem is NP-hard to solve in general. For example, in case of sparse regression, the obtained problem is $L_0$-constrained loss minimization which, in general, is not efficiently solvable \cite{Natarajan1995}. Similarly, in the case of low-rank regression, the obtained problem is rank-constrained minimization which again is NP-hard. 

Interestingly, recent works have demonstrated that these hardness results can be avoided by assuming certain natural conditions over the loss function being minimized such as restricted strong convexity (RSC) and restricted strong smoothness (RSS). The estimation routines proposed in these works typically make use of convex relaxations \cite{NegahbanRWY2012} or greedy methods \cite{LiuFY2014, JalaliJR2011, ShalevSZ2010} which do not suffer from infeasibility issues.

%Interestingly, recent works have demonstrated how these hardness results can be avoided by assuming that the loss function being minimized satisfies certain natural conditions such as restricted strong convexity/smoothness (RSC/RSS). These methods typically avoid the issue of non-convex optimization via a convex relaxation \cite{NegahbanRWY2012} or by using greedy methods \cite{LiuFY2014, JalaliJR2011, ShalevSZ2010}.

Despite this, certain limitations have precluded widespread use of these techniques. Convex relaxation-based methods typically suffer from slow rates as they solve non-smooth optimization problems apart from being hard to analyze in terms of global guarantees. Greedy methods, on the other hand, are slow in situations with non-negligible sparsity or relatively high rank, owing to their incremental approach of adding/removing individual support elements.

%Both approaches, however, have limitations that restrict their widespread use. Convex relaxation based methods typically solve non-smooth optimization problems that offer relatively slow rates of convergence as well as are hard to analyze in terms of global guarantees. Greedy methods, on the other hand suffer in situations with non-negligible sparsity or relatively high rank owing to their incremental approach of adding/removing individual support elements per iteration.

%However, both approaches have issues that make them relatively less used in practice. Convex relaxation based methods typically need to solve non-smooth optimization problems and have relatively slower rate of convergence and the analysis is also fairly involved as it requires one to understand global optima of a non-smooth function. Greedy methods select/remove only one element per iteration and hence are slow when the number of non-zeros in model or the rank of the model is relatively high. 

Instead, the methods of choice for practical applications are actually projected gradient (PGD) methods, also referred to as iterative hard thresholding (IHT) methods. These methods directly project the gradient descent update onto the underlying (non-convex) feasible set. This projection can be performed efficiently for several interesting structures such as sparsity and low rank. However, traditional PGD analyses for convex problems viz. \cite{Nesterov2004} do not apply to these techniques due to the non-convex structure of the problem.

%In contrast, the methods of choice for practical applications are actually projected gradient descent (PGD) methods, also referred to as iterative hard thresholding (IHT) methods. These methods project the gradient descent update directly onto the underlying (non-convex) feasible set. This projection can be efficiently done for several interesting structures including sparsity and low rank. However, due to the non-convexity in the structure of the problem, standard PGD analysis and rates from convex optimization \cite{Nesterov2004} cannot be applied to these methods. 

An exception to this is the recent work \cite{LohW2013} that demonstrates that PGD with non-convex regularization can offer consistent estimates for certain high-dimensional problems. However, the work in \cite{LohW2013} is only able to analyze penalties such as SCAD, MCP and capped $L_1$. Moreover, their framework cannot handle commonly used penalties such as $L_0$ or low-rank constraints.

\subsection{Insufficiency of RIP based Guarantees for M-estimation} 
As noted above, PGD/IHT-style methods have been very popular in literature for sparse recovery and several algorithms including Iterative Hard Thresholding (IHT) \cite{BlumensathD2009} or GraDeS \cite{GargK2009}, Hard Thresholding Pursuit (HTP) \cite{Foucart11}, CoSaMP \cite{NeedellT08}, Subspace Pursuit (SP) \cite{DaiM09}, and OMPR($\ell$) \cite{JainTD11} have been proposed. However, the analysis of these algorithms has traditionally been restricted to settings that satisfy the Restricted Isometry property (RIP) or incoherence property. As the discussion below demonstrates, this renders these analyses inaccessible to high-dimensional statistical estimation problems.

All existing results analyzing these methods require the condition number of the loss function, restricted to sparse vectors, to be smaller than a universal constant. The best known such constant is due to the work of \cite{JainTD11} that requires a bound on the RIP constant $\delta_{2k}\leq 0.5$ (or equivalently a bound $\tfrac{1+\delta_{2k}}{1-\delta_{2k}} \le 3$ on the condition number).
%In fact, existing results show that for some methods like $L_1$ cannot solve the problem optimally beyond condition number $\kappa\geq 6$
In contrast, real-life high dimensional statistical settings, wherein pairs of variables can be arbitrarily correlated, routinely require estimation methods to perform under arbitrarily large condition numbers. In particular if two variates have a covariance matrix like ${\left[\begin{matrix}1&1-\epsilon\\1-\epsilon&1\end{matrix}\right]}$, then the restricted condition number (on a support set of size just $2$) of the sample matrix cannot be brought down below $1/\epsilon$ even with infinitely many samples. In particular when $\epsilon < 1/6$, none of the existing results for hard thresholding methods offer \emph{any} guarantees. Moreover, most of these analyses consider only the least squares objective. Although recent attempts have been made to extend this to general differentiable objectives \cite{BahmaniRB2013,YuanLZ2014}, the results continue to require that the restricted condition number be less than a universal constant and remain unsatisfactory in a statistical setting.

\subsection{Overview of Results}
Our main contribution in this work is an analysis of PGD/IHT-style methods in statistical settings. Our bounds are tight, achieve known minmax lower bounds \cite{ZhangWJ14}, and hold for arbitrary differentiable, possibly even \emph{non-convex} functions. Our results hold even when the underlying condition number is arbitrarily large and only require the function to satisfy RSC/RSS conditions. In particular, this reveals that these iterative methods are indeed applicable to statistical
settings, a result that escaped all previous works.

%Our results only require that the functions satisfy the RSC/RSS conditions which reveals that these iterative methods are indeed applicable to statistical settings, a result that escaped all previous works. In particular, our analyses hold even when the underlying condition number is arbitrarily large.%, although the sample complexity goes up in those cases.

Our first result shows that the PGD/IHT methods achieve global convergence if used with a relaxed projection step. More formally, if the optimal parameter is $\so$-sparse and the problem satisfies RSC and RSS constraints $\sc$ and $\ss$ respectively (see Section~\ref{sec:prob}), then PGD methods offer global convergence so long as they employ projection to an $s$-sparse set where $s \geq 4(\ss/\sc)^2\so$. This gives convergence rates that are identical to those of convex relaxation and greedy methods for the Gaussian sparse linear model. We then move to a family of efficient ``fully corrective'' methods and show as before, that for arbitrary functions satisfying the RSC/RSS properties, these methods offer global convergence.

%In this work, we study the $L_0$ (and low-rank) constrained PGD method for optimizing differentiable, possibly \emph{non-convex} functions, as long as they satisfy RSC/RSS conditions. These conditions are already known to hold in many statistical settings. In particular they hold even when the condition number of the population covariance matrix is arbitrarily large (this does result in more samples being needed to satisfy the conditions).

%We show that the PGD method converges to the global optima, assuming that $L_0$ norm $\s$ of the estimator is larger than $4(\ss/\sc)^2\so$, where $\so$ is the $L_0$ norm of the optimal model parameter, $\ss$ is the RSS constant and $\sc$ is the RSC constant. This implies that, for Gaussian linear $\so$-sparse model, the output of PGD, has a $O(\sqrt{\tfrac{\so \log p}{n}})$ convergence rate to the optimal parameter (in $L_2$ norm). Note that our result matches the corresponding existing results for the convex relaxation methods as well as greedy methods. We also obtain a similar result  when the data is corrupted with additive noise. As~\cite{LohW2012} showed, removing the effects of noise leads to a \emph{non-convex} empirical objective. However, our PGD analyses recovers similar guarantees as theirs because RSC/RSM conditions still hold. Note that we run PGD on an $L_0$ constrained problem whereas they run it on an $L_1$ regularized problem.

Next, we show that these results allow PGD-style methods to offer global convergence in a variety of statistical estimation problems such as sparse linear regression and low rank matrix regression. Our results effortlessly extend to the noisy setting as a corollary and give bounds similar to those of \cite{LohW2012} that relies on solving an $L_1$ regularized problem.

%However, we show that by using more non-zeros (i.e., increase $\s$), we can show that the same methods actually can solve the problem optimally and the required samples achieve the minimax lower bounds.
Our proofs are able to exploit that even though hard-thresholding is not the prox-operator for any convex prox function, it still provides strong contraction when projection is performed onto sets of sparsity $\s \gg \so$. This crucial observation allows us to provide the first unified analysis for hard thresholding based gradient descent algorithms.
%, where the projection is onto the set of $\s$-sparse vectors and $\so$ is the number of non-zeros of the comparator point. Moreover, we provide an analysis for generic hard thresholding based gradient descent type methods and cast several of the existing methods in the same framework to provide first such analysis for all of them. %when compared to the optimal contraction provided by a $\so$-sparse vector and 
%A key difference of our analysis to the results of the above mentioned hard thresholding are significantly stronger than similar analysis for compressive sensing type settings. The key difference of our results to those 
Our empirical results confirm our predictions with respect to the recovery properties of IHT-style algorithms on badly-conditioned sparse recovery problems, as well as demonstrate that these methods can be orders of magnitudes faster than their $L_1$ and greedy counterparts.
\subsection{Organization of the Paper}
Section~\ref{sec:prob} sets the notation and the problem statement. Section~\ref{sec:algo} introduces the PGD/IHT algorithm that we study and proves that the method guarantees recovery assuming the RSC/RSS property. We also generalize our guarantees to the problem of low-rank matrix regression. Section~\ref{sec:stat} then provides crisp sample complexity bounds and statistical guarantees for the PGD/IHT estimators. Section~\ref{sec:fully} extends our analysis to a broad family of ``fully-corrective'' hard thresholding methods compressive sensing algorithms that include the so-called \emph{two-stage hard thresholding} and \emph{partial hard thresholding} algorithms and provide similar results for them as well. We present some empirical results in Section~\ref{sec:exp} and conclude in Section~\ref{sec:conc}.
%\subsection{Related Works}
%high-d recovery: $L_1$, pradeep's work, shai's work
%compressive sensing: grades, ompr
%discrete graphical models
%%% Local Variables: 
%%% mode: latex
%%% TeX-master: "rsc_iht"
%%% End: 

\section{Problem Setup and Notations}
\label{sec:prob}
{\bf High-dimensional Sparse Estimation.} Given data points $X=[X_1, \dots, X_n]^T$, where $X_i\in \R^p$, and the target $Y=[Y_1, \dots, Y_n]^T$, where $Y_i \in \R$, the goal is to compute an $\so$-sparse $\bto$ s.t., 
\begin{equation}
	\bto=\arg\min_{\bt, \|\bt\|_0\leq \so}f(\bt).\label{eq:prob_sp}
\end{equation}
Typically, $f$ can be thought of as an empirical risk function i.e. $f(\bt)=\frac{1}{n}\sum_i \ell(\ip{X_i}{\bt}, Y_i)$ for some loss function $\ell$ (see examples in Section~\ref{sec:stat}). However, for our analysis of PGD and other algorithms, we need not assume any other property of $f$ other than differentiability and the following two RSC and RSS properties.% [{\bf TODO}: Can we parameterize RSC/RSS with a single parameter $s=s_1+s_2$? It seems inelegant to use $s_1,s_2$. We can then require that $\|\bt_1\|_0+\|\bt_2\|_0\leq s =\s_1+\s_2$]
\begin{definition}[RSC Property]
\label{defn:rsc}
A differentiable function $f:\R^p\rightarrow \R$ is said to satisfy restricted strong convexity (RSC) at sparsity level $s = \s_1+\s_2$ with strong convexity constraint $\sc_{s}$ if the following holds for all $\bt_1, \bt_2$ s.t. $\|\bt_1\|_0\leq \s_1$ and $\|\bt_2\|_0\leq \s_2$: 
$$f(\bt_1)-f(\bt_2)\geq \ip{\bt_1-\bt_2}{\nabla_{\bt} f(\bt_2)}+\frac{\sc_{s}}{2}\|\bt_1-\bt_2\|_2^2.$$
\end{definition}
\begin{definition}[RSS Property]
\label{defn:rss}
A differentiable function $f:\R^p\rightarrow \R$ is said to satisfy restricted strong smoothness (RSS) at sparsity level $s = \s_1+\s_2$ with strong convexity constraint $\ss_{s}$ if the following holds for all $\bt_1, \bt_2$ s.t. $\|\bt_1\|_0\leq \s_1$ and $\|\bt_2\|_0\leq \s_2$:
$$f(\bt_1)-f(\bt_2)\leq \ip{\bt_1-\bt_2}{\nabla_{\bt} f(\bt_2)}+\frac{\ss_{s}}{2}\|\bt_1-\bt_2\|_2^2.$$ 
\end{definition}
{\bf Low-rank Matrix Regression.} Low-rank matrix regression is similar to sparse estimation as presented above except that each data point is now a matrix i.e. $X_i\in \R^{p_1\times p_2}$, the goal being to estimate a low-rank matrix $\W\in \R^{p_1\times p_2}$ that minimizes the empirical loss function on the given data.% As above, the problem is given by: 
\begin{equation}
  \label{eq:prob_mat}
  \Wo=\arg\min_{\W, rank(\W)\leq r}f(\W). 
\end{equation}
For this problem the RSC and RSS properties for $f$ are defined similarly as in Definition~\ref{defn:rsc}, \ref{defn:rss} except that the $L_0$ norm is replaced by the rank function. 
%\section{Algorithms}
%We first study the projected gradient method which is used heavily in practice. 
%We then study a fully corrective version of this algorithm and in fact provide a general framework in which we can study several of the existing compressive sensing type of iterative methods and provide guarantees for them. 

\section{Iterative Hard-thresholding Method}
\label{sec:algo}
In this section we study the popular projected gradient descent (a.k.a iterative hard thresholding) method for the case of the feasible set being the set of sparse vectors (see Algorithm~\ref{algo:iht} for pseudocode). The projection operator $P_s(\z)$, can be implemented efficiently in this case by projecting $\z$ onto the set of $\s$-sparse vectors by selecting the $\s$ largest elements (in magnitude) of $\z$. The standard projection property implies that $\|P_s(\z)-\z\|_2^2\leq \|\bt'-\z\|_2^2$ for all $\|\bt'\|_0\leq \s$.  However, it turns out that we can prove a significantly stronger property of hard thresholding for the case when $\|\bt'\|_0\leq \so$ and $\so\ll \s$. This property is key to analysing IHT and is formalized below.% We formalize the above mentioned observation in the following lemma, which is a key lemma in analysing IHT. 
\begin{lemma}\label{lem:ht1}
For any index set $I$, any $\z \in \R^{I}$, let $\bt = P_s(\z)$. Then for any $\bto \in \R^{I}$ such that $\|\bto\|_0\leq \so$, we have
$$\|\bt - \z\|_2^2\leq \frac{|I|-\s}{|I|-\so}\|\bto - \z\|_2^2.$$
\end{lemma}
%\begin{lemma}\label{lem:ht1}
%For any index set $\It$, any $\z_{\It}\in\R^{|\It|}$, let $\bt_{\It}=P_s(\z_{\It})$. Then for any $\bto_{\It}\in \R^{|\It|}$ such that $\|\bto_{\It}\|_0\leq \so$,
%$$\|\bt_{\It}-\z_{\It}\|_2^2\leq \frac{|\It|-\s}{|\It|-\so}\|\bto_{\It}-\z_{\It}\|_2^2.$$
%\end{lemma}
See Appendix~\ref{app:iht} for a detailed proof. 
\begin{algorithm}[t!]
\caption{Iterative Hard-thresholding}
  \begin{algorithmic}[1]
    \STATE {\bf Input}: Function $f$ with gradient oracle, sparsity level $s$, step-size $\eta$
    \STATE $\bt^1= \mathbf{0}$, $t = 1$
    \WHILE{ \emph{not converged} }
    \STATE $\ttn=P_s(\btt-\eta \nabla_{\bt}f(\btt))$
	\STATE $t = t+1$
    \ENDWHILE 
    \STATE {\bf Output}: $\btt$
  \end{algorithmic}\label{algo:iht}
\end{algorithm}

Our analysis combines the above observation with the RSC/RSS properties of $f$ to provide geometric convergence rates for the IHT procedure below. %Our main result is given by the below given theorem.%To demonstrate the key steps in our analysis and also to emphasise simplicity of the analysis, we first present our proof for the case when $\bto=\arg\min_{\bt, \|\bt\|_0\leq \so}f(\bt)$ is also the global optimizer, i.e, $\bto=\arg\min f(\bt)$ and consequently, $\nabla_{\bt} f(\bto)=0$.
%\begin{theorem}\label{thm:iht_simp}
%Let $f$, $\s$ be the input to Algorithm~\ref{algo:iht} [\textbf{TODO}: what about step-size $\eta$?]. Also, let the RSC parameter of $f$ is given by $\ss_{2\s+\so}(f)=\ss$ and similarly, RSC parameter is given by: $\sc_{2\s+\so}(f)=\sc$. Let $\s\geq 4\left(\frac{\ss}{\sc}\right)^2\so$.  Then, the $\tau$-th iterate of Algorithm~\ref{algo:iht} satisfies: 
%$$f(\bt^\tau)-f(\bto)\leq \left(1-\frac{\sc}{4\ss}\right)^\tau\cdot \left(f(\bt^0)-f(\bto)\right),$$
%where $\bto=\arg\min_{\bt, \|\bt\|_0\leq \so}f(\bt)$. That is, $\tau=O(\log(\frac{4\ss}{\sc}\cdot\frac{f(\bt^0)}{\epsilon})$-th iterate of Algorithm~\ref{algo:iht} satisfies: 
%$$f(\bt^\tau)-f(\bto)\leq \epsilon.$$
%\end{theorem}
\begin{theorem}\label{thm:iht_simp}
Let $f$ have RSC and RSS parameters given by $\ss_{2\s+\so}(f)=\ss$ and $\sc_{2\s+\so}(f)=\sc$ respectively. Let Algorithm~\ref{algo:iht} be invoked with $f$, $\s\geq 32\left(\frac{\ss}{\sc}\right)^2\so$ and $\eta = \frac{2}{3L}$. Also let $\bto=\arg\min_{\bt, \|\bt\|_0\leq \so}f(\bt)$. Then, the $\tau$-th iterate of Algorithm~\ref{algo:iht}, for $\tau=O(\frac{\ss}{\sc}\cdot\log(\frac{f(\bt^0)}{\epsilon}))$ satisfies:
$$f(\bt^\tau)-f(\bto)\leq \epsilon.$$
\end{theorem}
\begin{proof}(Sketch)
%  Our proof is inspired by the proof of the GradeS method by \cite{GargK08}, but is for general function $f$ and holds under general RSC condition. 
Let $\St=supp(\btt)$, $\So=supp(\bto)$, $\Stn=supp(\ttn)$ and $\It=\So\cup \St\cup \Stn$. %Note that, $\btt, \ttn, \bto$ all have support restricted to $\It$.
Using the RSS property and the fact that $supp(\btt)\subseteq \It$ and $supp(\ttn)\subseteq \It$, we have: 
\begin{align}
  f(\ttn)-f(\btt)&\leq \ip{\ttn-\btt}{\gt}+\frac{\ss}{2}\|\ttn-\btt\|_2^2,\nonumber\\
&= \frac{\ss}{2}\|\ttn_{\It}-\btt_{\It}+\frac{2}{3\ss} \cdot \gt_{\It}\|_2^2-\frac{1}{2 \ss}\|\gt_{\It}\|_2^2,\nonumber\\
&\stackrel{\zeta_1}{\leq} \frac{\ss}{2}\cdot \frac{|\It|-\s}{|\It|-\so}\cdot\|\bto_{\It}-\btt_{\It}+\frac{1}{\ss} \cdot \gt_{\It}\|_2^2-\frac{1}{2 \ss}(\|\gt_{\It\backslash (\St \cup \So)}\|_2^2+\|\gt_{\St\cup \So}\|_2^2), \label{eq:iht_2}
\end{align}
where $\zeta_1$ follows from an application of Lemma~\ref{lem:ht1} with $I = \It$ and the Pythagoras theorem.
 %$I=(\It\backslash (\St\cup \So))\cup (\St\cup \So) = \It$. 
The above equation has three critical terms. The first term can be bounded using the RSS condition. Using $f(\btt)-f(\bto)\leq \ip{\gt_{\St\cup \So}}{\btt-\bto}-\frac{\sc}{2}\|\btt-\bto\|_2^2\leq \frac{1}{2\sc}\|\gt_{\St\cup \So}\|_2^2$ bounds the third term in \eqref{eq:iht_2}. The second term is more interesting as in general elements of $\gt_{\overline{\So}}$ can be arbitrarily small. However, elements of $\gt_{\It\backslash(\St\cup \So)}$ should be at least as large as $\gt_{\So\backslash \Stn}$ as they are selected by hard-thresholding. Combining this insight with bounds for $\gt_{\So\backslash \Stn}$ and with \eqref{eq:iht_2}, we obtain the theorem. See Appendix~\ref{app:iht} for a detailed proof. %$\It\backslash (\St\cup \So)=\Stn\backslash (\St\cup \So)$
\end{proof}

\subsection{Low-rank Matrix Regression}
% \begin{algorithm}
% \caption{Iterative Hard-thresholding for Matrix Regression}
%   \begin{algorithmic}[1]
%     \STATE {\bf Input}: $X=[\X_1, \dots, \X_n]^T$, where $\X_i \in \R^{p_1\times p_2}$, $\by\in \R^{n}$, $\s$, $\eta$
%     \STATE $\W^1=0$
%     \FORALL{$t=1\dots$}
%     \STATE $\Wtn=P_s(\Wt-\eta \nabla_{\W}f(\Wt))$ 
%     \ENDFOR 
%     \STATE Output: $\Wtn$
%   \end{algorithmic}\label{algo:iht_matrix}
% \end{algorithm}
We now generalize our previous analysis to a projected gradient descent (PGD) method for low-rank matrix regression. Formally, we study the following problem: 
\begin{equation}
  \label{eq:mr_prob}
  \min_{\W}f(\W),\ s.t.,\ rank(\W)\leq \s.
\end{equation}
The hard-thresholding projection step for low-rank matrices can be solved using SVD i.e. 
$$PM_{\s}(\W)=U_\s\Sigma_\s V_\s^T,$$
where $\W=U\Sigma V^T$ is the singular value decomposition of $\W$. $U_s, V_s$ are the top-$\s$ singular vectors (left and right, respectively) of $\W$ and $\Sigma_\s$ is the diagonal matrix of the top-$s$ singular values of $\W$. To proceed, we first note a property of the above projection similar to Lemma~\ref{lem:ht1}.
\begin{lemma}\label{lem:svd1}
Let $\W\in \R^{p_1\times p_2}$ be a rank-$|\It|$ matrix and let $p_1\geq p_2$. Then for any rank-$\so$ matrix $\Wo\in \R^{p_1\times p_2}$ we have
\begin{equation}\label{eq:PMdef}
 \|PM_{\s}(\W)-\W\|_F^2\leq \frac{|\It|-\s}{|\It|-\so}\|\Wo-\W\|_F^2.
 \end{equation}
\end{lemma}
\begin{proof}
Let $\W=U\Sigma V^T$ be the singular value decomposition of $\W$. Now, $\|PM_{\s}(\W)-\W\|_F^2=\sum_{i=\s+1}^{|\It|}\sigma_{i}^2=\|P_s(diag(\Sigma))-diag(\Sigma)\|_2^2$, where $\sigma_1\geq \dots\geq \sigma_{|\It|}\geq 0$ are the singular values of $\W$. Using Lemma~\ref{lem:ht1}, we get: 
\begin{equation}
  \label{eq:svd1_1}
  \|PM_{\s}(\W)-\W\|_F^2\leq \frac{|\It|-\s}{|\It|-\so}\|\Sigma^*-diag(\Sigma)\|_2^2\leq \frac{|\It|-\s}{|\It|-\so}\|\Wo-\W\|_F^2,
\end{equation}
where the last step uses the von Neumann's trace inequality ($Tr(A\cdot B)\leq \sum_i \sigma_i(A)\sigma_i(B)$). 
\end{proof}
The following result for low-rank matrix regression immediately follows from Lemma~\ref{eq:mr_prob}.
%\begin{theorem}\label{thm:iht_matrix}
%Let $f$, $\s$ be supplied to Algorithm~\ref{algo:iht}, with projection $P_s$ replaced by its matrix counterpart $PM_s$ as defined in~\eqref{eq:PMdef}. Also, let the RSC parameter of $f$ is given by $\ss_{2\s+\so}(f)=\ss$ and similarly, RSC parameter is given by: $\sc_{2\s+\so}(f)=\sc$. Let $\s\geq 4\left(\frac{\ss}{\sc}\right)^2\so$. Then, the $\tau$-th iterate of Algorithm~\ref{algo:iht} satisfies: 
%$$f(\W^\tau)-f(\Wo)\leq \left(1-\frac{\sc}{4\ss}\right)^\tau\cdot \left(f(\W^0)-f(\Wo)\right),$$
%where $\Wo=\arg\min_{\W, rank(\W)\leq \so}f(\W)$. That is, $\tau=O(\log(\frac{4\ss}{\sc}\cdot\frac{f(\W^0)}{\epsilon})$-th iterate of Algorithm~\ref{algo:iht} satisfies: 
%$$f(\W^\tau)-f(\Wo)\leq \epsilon.$$
%\end{theorem}
\begin{theorem}\label{thm:iht_matrix}
Let $f$ have RSC and RSS parameters given by $\ss_{2\s+\so}(f)=\ss$ and $\sc_{2\s+\so}(f)=\sc$. Replace the projection operator $P_s$ in Algorithm~\ref{algo:iht} with its matrix counterpart $PM_s$ as defined in~\eqref{eq:PMdef}. Suppose we invoke it with $f,\s\geq 32\left(\frac{\ss}{\sc}\right)^2\so,\eta=\frac{2}{3L}$. Also let $\Wo=\arg\min_{\W, rank(\W)\leq \so}f(\W)$. Then the $\tau$-th iterate of Algorithm~\ref{algo:iht}, for $\tau=O(\frac{\ss}{\sc}\cdot\log(\frac{f(\W^0)}{\epsilon})$ satisfies:
$$f(\W^\tau)-f(\Wo)\leq \epsilon.$$
\end{theorem}
\begin{proof}
A proof progression similar to that of Theorem~\ref{thm:iht_simp} suffices. The only changes that need to be made are: firstly Lemma~\ref{lem:svd1} has to be invoked in place of Lemma~\ref{lem:ht1}. Secondly, in place of considering vectors restricted to a subset of coordinates viz. $\bt_{\S},\gt_{I}$, we would need to consider matrices restricted to subspaces i.e. $\W_S = U_SU_S^T\W$ where $U_S$ is a set of singular vectors spanning the range-space of $S$. 
\end{proof}

%%% Local Variables: 
%%% mode: latex
%%% TeX-master: "rsc_iht"
%%% End: 

\section{High Dimensional Statistical Estimation}
\label{sec:stat}

\def\loss{\mathcal{L}}
\def\btbar{\bar{\bt}}

This section elaborates on how the results of the previous section can be used to give guarantees for IHT-style techniques in a variety of statistical estimation problems. We will first present a generic convergence result and then specialize it to various settings. Suppose we have a sample of data points $Z_{1:n}$ and a loss function $\loss(\bt;Z_{1:n})$ that depends on a parameter $\bt$ and the sample. Then we can show the following result. (See Appendix~\ref{app:stat} for a proof.)
 %Then we have the following result. 

\begin{theorem}\label{thm:genstat}
Let $\btbar$ be any $\so$-sparse vector. Suppose $\loss(\bt;Z_{1:n})$ is differentiable and satisfies RSC and RSS at sparsity level $\s+\so$ with parameters $\sc_{\s+\so}$ and $\ss_{\s+\so}$ respectively, for $\s \ge 32\left( \tfrac{\ss_{2\s+\so}}{\sc_{2\s+\so}} \right)^2 \so$. Let $\bt^{\tau}$ be the $\tau$-th iterate of Algorithm~\ref{algo:iht} for $\tau$ chosen as in Theorem~\ref{thm:iht_simp} and $\eps$ be the function value error incurred by Algorithm~\ref{algo:iht}. Then we have
\[
\| \btbar - \bt^\tau \|_2 \le \frac{2\sqrt{\s+\so} \| \nabla \loss(\btbar;Z_{1:n}) \|_\infty}{\sc_{\s+\so}}
+ \sqrt{ \frac{2\epsilon}{\sc_{\s+\so}} } .
\]
\end{theorem}
%\begin{proof} See Appendix~\ref{app:stat}.
%\end{proof}
Note that the result does \emph{not} require the loss function to be convex. This fact will be crucially used later. We now apply the above result to several statistical estimation scenarios.\\

\noindent
\textbf{Sparse Linear Regression.}
Here $Z_i = (X_i,Y_i) \in \R^p \times \R$ and $Y_i = \ip{\btbar}{X_i} + \xi_i$ where $\xi_i \sim \mathcal{N}(0,\sigma^2)$ is label noise. The empirical loss is the usual least squares loss i.e. $\loss(\bt;Z_{1:n}) = \tfrac{1}{n}\|Y - X \bt\|_2^2$. Suppose $X_{1:n}$ are drawn i.i.d. from a sub-Gaussian distribution with covariance $\Sigma$ with $\Sigma_{jj} \le 1$ for all $j$. Then \cite[Lemma 6]{AgarwalNW2012} immediately implies that RSC and RSS at sparsity level $k$ hold, with probability at least $1-e^{-c_0n}$, with $\sc_{k} = \tfrac{1}{2}\sigma_{\min}(\Sigma) - c_1 \tfrac{k \log p}{n}$ and $\ss_{k} = 2 \sigma_{\max}(\Sigma) + c_1 \tfrac{k \log p}{n}$ ($c_0,c_1$ are universal constants). So we can set $k = 2\s+\so$ and if $n > 4c_1 k \log p/\sigma_{\min}(\Sigma)$ then
we have $\sc_{k} \ge \tfrac{1}{4} \sigma_{\min}(\Sigma)$ and $\ss_{k} \le 2.25 \sigma_{\max}(\Sigma)$ which means that $L_k/9\alpha_k \le  \kappa(\Sigma) := \sigma_{\max}(\Sigma)/\sigma_{\min}(\Sigma)$. Thus it is enough to choose $s = 2592 \kappa(\Sigma)^2 \so$ and apply Theorem~\ref{thm:genstat}. Note that $\| \nabla \loss(\btbar;Z_{1:n}) \|_\infty = \| X^T \xi / n \|_\infty \le 2\sigma\sqrt{\frac{\log p}{n}}$ with probability at least
$1-c_2p^{-c_3}$ ($c_2,c_3$ are universal constants). Putting everything together, we have the following bound with high probability:
\[
\| \btbar - \bt^\tau \|_2 \le 145 \frac{\kappa(\Sigma)}{\sigma_{\min}(\Sigma)} \sigma \sqrt{\frac{\so \log p}{n}}
+ 2  \sqrt{ \frac{\epsilon}{\sigma_{\min}(\Sigma)} },
\]
where $\epsilon$ is the function value error incurred by Algorithm~\ref{algo:iht}.\\

\noindent
\textbf{Noisy and Missing Data.}
We now look at cases with feature noise as well. More specifically, assume that we only have access to $\tilde{X}_i$'s that are corrupted versions of $X_i$'s. Two models of noise are popular in literature \cite{LohW2012}: a) (\emph{additive noise}) $\tilde{X}_i = X_i + W_i$ where $W_i \sim \mathcal{N}(\mathbf{0},\Sigma_W)$, and b) (\emph{missing data}) $\tilde{X}$ is an $\R \cup \{\star\}$-valued matrix obtained by independently, with probability $\nu\in [0,1)$, replacing each entry in $X$ with $\star$. For the case of additive noise (missing data can be handled similarly), $Z_i = (\tilde{X}_i,Y_i)$ and $\loss(\bt;Z_{1:n}) = \tfrac{1}{2}\bt^T \hat{\Gamma} \bt - \hat{\gamma}^T \bt$ where $\hat{\Gamma} = \tilde{X}^T \tilde{X}/n - \Sigma_W$ and $\hat{\gamma} = \tilde{X}^T Y/n$ are unbiased estimators of $\Sigma$ and $\Sigma^T \btbar$ respectively. \cite[Appendix A, Lemma 1]{LohW2012} implies that RSC, RSS at sparsity level $k$ hold, with failure probability exponentially small in $n$, with $\sc_{k} = \tfrac{1}{2}\sigma_{\min}(\Sigma) - k \tau(p)/n$ and $\ss_{k} = 1.5\sigma_{\max}(\Sigma) + k \tau(p)/n$ for $\tau(p) = c_0 \sigma_{\min}(\Sigma) \max( \tfrac{(\|\Sigma\|_{\mathrm{op}}^2 + \|\Sigma_W\|_{\mathrm{op}}^2)^2}{\sigma^2_{\min}(\Sigma)}, 1 ) \log p$.
Thus for $k = 2\s+\so$ and $n \ge 4k\tau(p)/\sigma_{\min}(\Sigma)$ we have $L_k/\alpha_k \le 7\kappa(\Sigma)$. Note that $\loss(\cdot; Z_{1:n})$ is \emph{non-convex} but we can still apply Theorem~\ref{thm:genstat} with $\s = 1568 \kappa(\Sigma)^2 \so$ because RSC, RSS hold. Using the high probability upper bound (see \cite[Appendix A, Lemma 2]{LohW2012}) $\| \nabla \loss(\btbar;Z_{1:n}) \|_\infty \le c_1 \tilde{\sigma} \| \btbar \|_2 \sqrt{\log p/n}$ gives us the following 
\[
\| \btbar - \bt^\tau \|_2 \le c_2 \frac{\kappa(\Sigma)}{\sigma_{\min}(\Sigma)} \tilde{\sigma} \| \btbar \|_2 \sqrt{\frac{\so \log p}{n}}
+ 2  \sqrt{ \frac{\epsilon}{\sigma_{\min}(\Sigma)} }
\]
where $\tilde{\sigma} = \sqrt{\|\Sigma_W\|^2_{\mathrm{op}} + \|\Sigma\|^2_{\mathrm{op}} }(\|\Sigma_W\|_{\mathrm{op}} + \sigma)$ and $\epsilon$ is the function value error in Algorithm~\ref{algo:iht}.

%\noindent
%\textbf{Generalized Linear Models.}
%Here $Z_i = (X_i,Y_i) \in \R^p \times \R$ are the samples and $\mathbb{P}(y|x;\btbar) \propto \exp\left( y \ip{\btbar}{ x} - \psi(\btbar^T x) \right)$ is an exponential family distribution such that $\psi''$ is uniformly upper bounded. The negative (conditional) log-likelihood $\loss(\bt;Z_{1:n}) = -\tfrac{1}{n} \sum_{i=1}^n \log \mathbb{P}(Y_i|X_i;\bt)$ is our loss. Assuming $\|\btbar\|_0 \le \so$, it is easy to see that this formulation includes high dimensional sparse linear, logistic, and multinomial (but not Poisson) regressions as special cases. Because RSC, RSS conditions have already been established for GLMs (see Proposition 1 in \cite{LohW2013}), it should be possible to extend our sparse linear regression results above to the sparse GLM case.

\section{Fully-corrective Methods}
\label{sec:fully}
In this section, we study a variety of ``fully-corrective" methods. These methods keep the optimization objective fully minimized over the support of the current iterate. To this end, we first prove a fundamental theorem for fully-corrective methods that formalizes the intuition that for such methods, a large function value should imply a large gradient at any sparse $\bt$ as well. This result is similar to Lemma 1 of \cite{JainTD11} but holds under RSC/RSS conditions (rather than the RIP condition as in \cite{JainTD11}), as well as for the general loss functions.

\begin{lemma}\label{lem:full_1}
Consider a function $f$ with RSC parameter given by $\ss_{2\s+\so}(f)=\ss$ and RSS parameter given by $\sc_{2\s+\so}(f)=\sc$. Let $\bto=\arg\min_{\bt, \|\bt\|_0\leq \so} f(\bt)$ with $\So=supp(\bto)$. Let $\St\subseteq [p]$ be any subset of co-ordinates s.t. $|\St|\leq \s$. Let $\btt=\arg\min_{\bt, supp(\bt)\subseteq \St}f(\bt)$. Then, we have: 
$$2\sc(f(\btt)-f(\bto))\leq \|\gt_{\St\cup \So}\|_2^2-\sc^2\|\btt_{\St\backslash \So}\|_2^2$$
\end{lemma}
%\begin{lemma}\label{lem:full_1}
%Let the RSC parameter of $f$ be given by $\ss_{2\s+\so}(f)=\ss$ and similarly, RSC parameter is given by $\sc_{2\s+\so}(f)=\sc$. Let $\bto=\arg\min_{\bt, \|\bt\|_0\leq \so} f(\bt)$ and let $\So=supp(\bto)$. Let $\gamma\geq 0$ be any constant s.t. $\gamma\geq \frac{1}{\sc}$. Let $\St\subseteq [p]$ be any subset of co-ordinates s.t. $|\St|\leq \s$. Let $\btt=\arg\min_{\bt, supp(\bt)\subseteq \St}f(\bt)$. Then, the following holds: 
%$$2\gamma(f(\btt)-f(\bto))\leq 2\gamma\left(f(\btt)-f(\bto)+\frac{\sc}{2}\cdot\left(1-\frac{1}{\sc\gamma}\right)\|\btt-\bto\|_2^2\right)\leq \gamma^2\|\gt_{\St\cup \So}\|_2^2-\|\btt_{\St\backslash \So}\|_2^2,$$
%where $\gt=\nabla_{\bt} f(\btt)$. In particular, selecting $\gamma=\frac{1}{\sc}$, we have: 
%$$2\sc(f(\btt)-f(\bto))\leq \|\gt_{\St\cup \So}\|_2^2-\sc^2\|\btt_{\St\backslash \So}\|_2^2$$
%\end{lemma}
See Appendix~\ref{app:fully} for a detailed proof.

\begin{algorithm}[t!]
\caption{Two-stage Hard-thresholding}
  \begin{algorithmic}[1]
    \STATE {\bf Input}: function $f$ with gradient oracle, sparsity level $\s$, sparsity expansion level $\ell$
    \STATE $\bt^1=0$, $t=1$
    \WHILE{\emph{not converged}}
    \STATE $\gt=\nabla_{\bt}f(\btt)$, $\St=supp(\btt)$
    \STATE $Z^t=\St\cup (\text{largest } \ell \text{ elements of }|\gt_{\overline{\St}}|)$
    \STATE $\bbt=\arg\min_{\bm{\beta}, supp(\beta)\subseteq Z^t}f(\bm{\beta})$ \hfill // fully corrective step%
    \STATE $\widetilde{\bt}^t=P_\s(\bbt)$
    \STATE $\ttn=\arg\min_{\bt, supp(\bt)\subseteq supp(\widetilde{\bt}^t)}f(\bt)$ \hfill // fully corrective step
    \STATE $t=t+1$
	\ENDWHILE 
    \STATE {\bf Output}: $\btt$
  \end{algorithmic}\label{algo:tstage}
\end{algorithm}

\subsection{Two-stage Hard Thresholding Methods}
Here we will concentrate on a family of two-stage fully corrective methods that contains popular compressive sensing algorithms like CoSaMP and Subspace Pursuit (see Algorithm~\ref{algo:tstage} for pseudocode). These algorithms have thus far been analyzed only under RIP conditions for the least squares objective. Using our analysis framework developed in the previous sections, we present a generic RSC/RSS-based analysis for general two-stage methods for arbitrary loss functions. Our analysis shall use the following key observation that the the hard thresholding step in two stage methods does not increase the objective function a lot.

%In this section, we study the family of two-stage methods which contains  Both these algorithms have been analyzed for least squares objective under RIP guarantees. Here, using an analysis similar to the one developed in the last section, we present a generic RSC based analysis for the general two-stage family (for optimizing any function $f$). 

\begin{lemma}\label{lem:ht2} Let $Z_t\subseteq [n]$ and $|Z_t|\leq \sa$. Let $\bbt=\arg\min_{\bb, supp(\bb)\subseteq Z_t} f(\bb)$ and $\bwt=P_\sa(\bbt)$. Then, the following holds: 
$$f(\bwt)-f(\bbt)\leq \frac{\ss}{\sc}\cdot \frac{\ell }{\s+\ell-\so }\cdot(f(\bbt)-f(\bto)).$$
\end{lemma}
\begin{proof}
Let $\vt=\nabla_{\bt} f(\bbt)$. Then, using the RSS property we get:
  \begin{align}
    f(\bwt)-f(\bbt)&\leq \ip{\bwt-\bbt}{\vt}+\frac{\ss}{2}\|\bwt-\bbt\|_2^2\stackrel{\zeta_1}{=}\frac{\ss}{2}\|\bwt-\bbt\|_2^2\stackrel{\zeta_2}{\leq} \frac{\ss}{2} \frac{|\ell |}{|\s+\ell-\so |}\|w-\bbt\|_2^2,\label{eq:ht2_1}
%&=\frac{\ss}{2}\|\bbt_{Z_t\backslash supp(\bwt)}\|_2^2\stackrel{\zeta_2}{\leq}\frac{\ss}{2}\|\bbt_{Z_t\backslash \So}\|_2^2,
  \end{align}
where $w$ is any vector such that $w_{\overline{Z_t}}=0$ and $\|w\|_0\leq \so$.  $\zeta_1$ follows by observing $\vt_{Z_t}=0$ and by noting that $supp(\bwt)\subseteq Z_t$. $\zeta_2$ follows by Lemma~\ref{lem:ht1} and the fact that $\|w\|_0\leq \so$. Now, using RSS property and the fact that $\nabla_{\bt}f(\bbt)=0$, we have: 
\begin{align}
\frac{\sc}{2}\|w-\bbt\|_2^2\leq f(\bbt)-f(w)\leq f(\bbt)-f(\bto). 
  \label{eq:ht2_2}
\end{align}
The result now follows by combining \eqref{eq:ht2_1} and \eqref{eq:ht2_2}. 
\end{proof}

%\begin{theorem}
%Let $f$, $\s$, $\ell\geq \so$ be supplied to Algorithm~\ref{algo:tstage}. Also, let the RSC parameter of $f$ is given by $\ss_{2\s+\ell}(f)=\ss$ and similarly, RSC parameter is given by: $\sc_{2\s+\so}(f)=\sc$. Also, let $\s\geq 4\frac{\ss^2}{\sc^2}\ell+\so-\ell\geq 4\frac{\ss^2}{\sc^2}\so$. Then, the $\tau$-th iterate of Algorithm~\ref{algo:tstage} satisfies: 
%$$f(\bt^\tau)-f(\bto)\leq \left(1-\frac{\sc^2}{\ss^2}\right)^\tau\cdot \left(f(\bt^0)-f(\bto)\right),$$
%where $\bto=\arg\min_{\bt, \|\bt\|_0\leq \so}f(\bt)$. That is, $\tau=O\left(\frac{4\ss}{\sc}\cdot\log(\frac{f(\bt^0)}{\epsilon})\right)$-th iterate of ALgorithm~\ref{algo:pht} satisfies: 
%$$f(\bt^\tau)-f(\bto)\leq \epsilon.$$\label{thm:tstage} 
%\end{theorem}
\begin{theorem}
\label{thm:tstage}
Let $f$ has RSC and RSS parameters given by $\sc_{2\s+\so}(f)=\sc$ and $\ss_{2\s+\ell}(f)=\ss$ resp. Call Algorithm~\ref{algo:tstage} with $f$, $\ell\geq \so$ and $\s\geq 4\frac{\ss^2}{\sc^2}\ell+\so-\ell\geq 4\frac{\ss^2}{\sc^2}\so$. Also let $\bto=\arg\min_{\bt, \|\bt\|_0\leq \so}f(\bt)$. Then, the $\tau$-th iterate of Algorithm~\ref{algo:tstage}, for $\tau = O(\frac{\ss}{\sc}\cdot\log(\frac{f(\bt^0)}{\epsilon})$ satisfies:
$$f(\bt^\tau)-f(\bto)\leq \epsilon.$$
\end{theorem}
See Appendix~\ref{app:fully} for a detailed proof.

\subsection{Partial Hard Thresholding Methods}
\label{sec:pht}
\begin{algorithm}
\caption{Iterative Partial Hard-thresholding}
  \begin{algorithmic}[1]
    \STATE {\bf Input}: function $f$ with gradient oracle, sparsity level $\s$, step size $\eta$, partial thresholding level $\ell$
    \STATE $\bt^1=0$, $t=1$
    \WHILE{\emph{not converged}}
    \STATE $\z^t=\btt-\eta \nabla_{\btt}f(\btt)$, $\St=supp(\btt)$
    \STATE $\vt=\text{PHT}_{\s}(\z^t; \St, \ell)$
	\STATE $\ttn = \arg\min_{\bt, supp(\bt)\subseteq supp(\vt)}f(\bt)$ \hfill // fully corrective step
    \STATE $t=t+1$
	\ENDWHILE
    \STATE {\bf Output}: $\btt$
  \end{algorithmic}\label{algo:pht}
\end{algorithm}
We now study Partial Hard Thresholding methods (PHT), a family of fully-corrective iterative methods introduced by \cite{JainTD11}. This family is known to provide the best known RIP guarantees in the compressive sensing setting, but the proof is restricted to the RIP setting, and for the least-squares objective. An interesting member of this family is Orthogonal Matching Pursuit with Replacement (OMPR), which is also a Forward-Backward Greedy Selection method but performs one forward-backward step per iteration. 

The pseudo code of the general IPHT$(\ell)$ algorithm is given in Algorithm~\ref{algo:pht}. The algorithm is similar to the fully-corrective projected gradient descent (PGD) method, in fact, PHT$(0)$ is indeed exactly same as the fully-corrected PGD method. But, the partial hard-thresholding projection is used instead of hard-thresholding projection. 

The partial hard thresholding operator $\vt=\text{PHT}_{\s}(\z^t; \St, \ell)$ projects $\z^t$ onto the non-convex set of $\s$-sparse vectors s.t. $|supp(\vt)\backslash\St|\leq \ell$. Although, the operator projects onto a non-convex set, still the projection can be performed efficiently by performing hard thresholding only over $\z^t_{\overline{\St}}$ and the $\ell$ smallest elements of $\z^t_{\St}$. That is, let $bot^t = $ smallest $\ell$ elements of $|\z^t_{\St}|$. Then, $$\vt_{\St\backslash bot^t}=\z^t_{\St\backslash bot^t},\ \text{and},\ \vt_{\overline{\St}\cup bot^t}=H_\ell(\z^t_{\overline{\St}\cup bot^t}).$$
We first show that at least a new element is added during each iteration of IPHT($\ell$). 
\begin{lemma}\label{lem:pht_new}
Let $f$, $\s$ be supplied to Algorithm~\ref{algo:pht} and let the RSC and RSS parameters of $f$ be given by $\ss_{2\s+\so}(f)=\ss$ and $\sc_{2\s+\so}(f)=\sc$ respectively. Let $\s\geq 4\left(\frac{\ss}{\sc}\right)^2\so$. Then, either $f(\St)=f(\So)$ or $|\Stn\backslash \St|\geq 1$. That is, at least one new element is added at each iteration of Algorithm~\ref{algo:pht}. 
\end{lemma}
\begin{proof}
Suppose no new element is added i.e. $|\Stn\backslash \St|=0$. Since $\ell$ elements of $\Stn$ are selected by hard thresholding $\z^t_{\overline{\St}\cup bot^t}$, hence each element of $\z^t_{\St}$ should be larger (in magnitude) than $\max_{i\in \overline{\St}}|z^t_i|$. Note that, $\z^t_{\overline{\St}}=-\eta \gt_{\overline{\St}}$. Hence, 
\begin{align*}
\frac{\|\z^t_{\St\backslash \So}\|_2^2}{|\St\backslash \So|} = \frac{\|\btt_{\St\backslash \So} - \eta\gt_{\St\backslash \So}\|_2^2}{|\St\backslash \So|} = 
\frac{\|\btt_{\St\backslash \So}\|_2^2}{|\St\backslash \So|} &\geq \frac{\|\z^t_{\So\backslash \St}\|_2^2}{|\So\backslash \St|} = \frac{\|\btt_{\So\backslash \St} - \eta\gt_{\So\backslash \St}\|_2^2}{|\So\backslash \St|} \\ &=\eta^2\frac{\|\gt_{\So\backslash \St}\|_2^2}{|\So\backslash \St|}=\eta^2\frac{\|\gt_{\So\cup \St}\|_2^2}{|\So\backslash \St|},
\end{align*}
where we have used the fact that $\gt_{\St}=0$ and $\btt_{\overline{\St}} = 0$.

Using Lemma~\ref{lem:full_1} and the above equation, we have: 
\begin{align}
  0\leq 2\gamma\left(f(\btt)-f(\bto)-\frac{\sc}{2}\cdot\left(\frac{1}{\sc\gamma}-1\right)\|\btt-\bto\|_2^2\right)\leq \left(\gamma^2-\eta^2\frac{|\St\backslash \So|}{|\So\backslash \St|}\right)\|\gt_{\St\cup \So}\|^2.
\end{align}
The lemma now follows by observing that $\frac{|\St\backslash \So|}{|\So\backslash \St|}\geq \frac{s - \so}{\so} \geq \frac{\gamma^2}{\eta^2}\geq \frac{1}{\eta^2\sc^2}$, by the choice of $s$.
\end{proof}

We now provide the proof of convergence for IPHT($\ell$) method in the general RSC-RSS setting: 
\begin{theorem}
\label{thm:pht}
Let $f, s$ be supplied to Algorithm~\ref{algo:pht}. Also, let the RSC and RSS parameters of $f$ be given by $\sc_{2\s+\so}(f)=\sc$ and $\ss_{2\s+\so}(f)=\ss$ respectively. Let $\s\geq 4\left(\frac{\ss}{\sc}\right)^2\so$ and let $\eta=\frac{1}{2\ss}$. Then, the $\tau$-th iterate of Algorithm~\ref{algo:pht} satisfies: 
$$f(\bt^\tau)-f(\bto)\leq \left(1-\frac{\sc}{4\ss}\cdot \frac{1}{\ell+1}\right)^\tau\cdot \left(f(\bt^0)-f(\bto)\right),$$
where $\bto=\arg\min_{\bt, \|\bt\|_0\leq \so}f(\bt)$.
\end{theorem}
This implies that for $\tau=O\left(\frac{\ss\ell}{\sc}\cdot\log(\frac{f(\bt^0)}{\epsilon})\right)$, we have $f(\bt^\tau)-f(\bto)\leq \epsilon$. See Appendix~\ref{app:fully} for a detailed proof.

\section{Experiments}
\label{sec:exp}
\begin{figure*}[t]
               \centering
			   \hspace*{-1ex}
               \subfigure[\!\!\!\!\!\!\!\!]{
                              \includegraphics[scale=0.6]{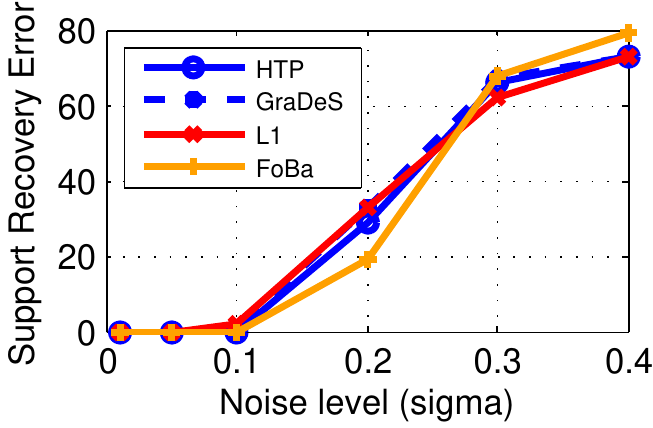}
                              \label{subfig:noise}
               }\hspace*{-1ex}
               \subfigure[\!\!\!\!\!\!\!\!]{
                              \includegraphics[scale=0.6]{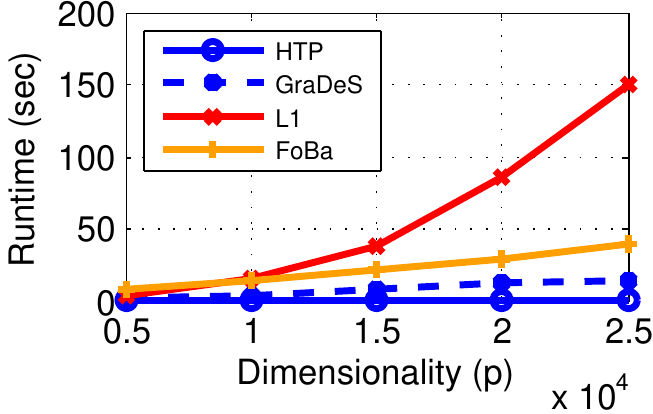}
                              \label{subfig:dimen}
               }\hspace*{-1ex}           
               \subfigure[\!\!\!\!\!\!\!\!]{
               				
                              \includegraphics[scale=0.6]{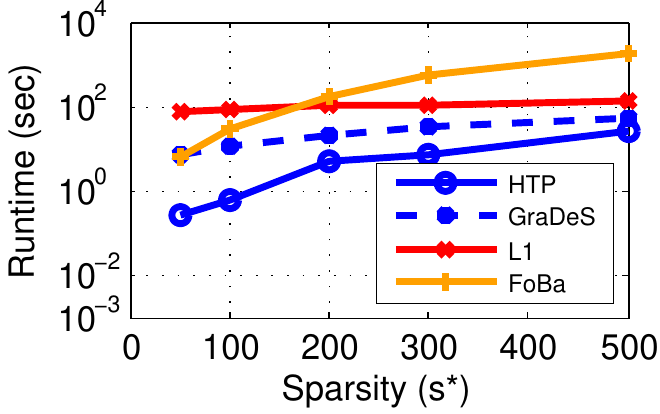}
                              \label{subfig:sparsity}
               }\hspace*{-1ex}
               \subfigure[\!\!\!\!\!\!\!\!]{
               				
                              \includegraphics[scale=0.6]{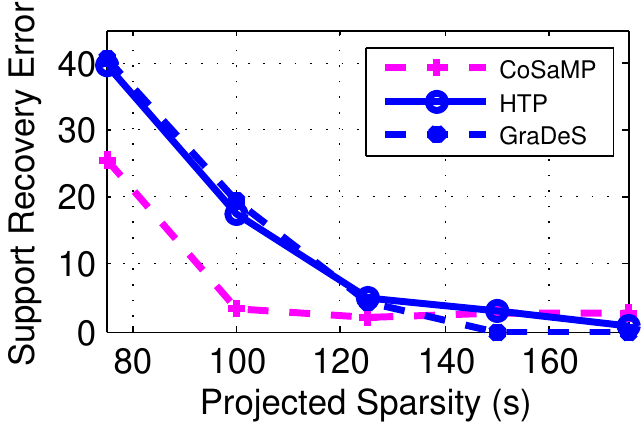}
                              \label{subfig:oversamp}
               }
               %\vspace*{-2ex}
               \caption{A comparison of hard thresholding techniques (HTP) and projected gradient methods (GraDeS) with L1 and greedy methods (FoBa) on sparse noisy linear regression tasks. \ref{subfig:noise} gives the number of undiscovered elements from $supp(\bto)$ as label noise levels are increased. \ref{subfig:dimen} shows the variation in running times with increasing dimensionality $p$. \ref{subfig:sparsity} gives the variation in running times (in logscale) when the true sparsity level $\so$ is increased keeping $p$ fixed. HTP and GraDeS are clearly much more scalable than L1 and FoBa. \ref{subfig:oversamp} shows the recovery properties of different IHT methods under large condition number ($\kappa = 50$) setting as the size of projected set is increased. 
               %\vspace*{-1ex}
				}
               \label{fig:expts}
\end{figure*}
%effect of increasing sample sizes relative to the base value $\so\cdot\log p$.

We conducted simulations on high dimensional sparse linear regression problems to verify our predictions. Our experiments demonstrate that hard thresholding and projected gradient techniques can not only offer recovery in stochastic setting, but offer much more scalable routines for the same.\\

\textbf{Data}: Our problem setting is identical to the one described in the previous section. We fixed a parameter vector $\btbar$ by choosing $\so$ random coordinates and setting them randomly to $\pm 1$ values. Data samples were generated as $Z_i = (X_i,Y_i)$ where $X_i \sim \mathcal{N}(0,I_p)$ and $Y_i = \ip{\btbar}{X_i} + \xi_i$ where $\xi_i \sim \mathcal{N}(0,\sigma^2)$. We studied the effect of varying dimensionality $p$, sparsity $\so$, sample size $n$ and label noise level $\sigma$ on the recovery properties of the various algorithms as well as their run times. We chose baseline values of $p = 20000, \so = 100, \sigma = 0.1, n = f_o\cdot\so\log p$ where $f_o$ is the oversampling factor with default value $f_o = 2$. Keeping all other quantities fixed, we varied one of the quantities and generated independent data samples for the experiments.\\

\textbf{Algorithms}: We studied a variety of hard-thresholding style algorithms including HTP \cite{Foucart11}, GraDeS \cite{GargK2009} (or IHT \cite{BlumensathD2009}), CoSaMP \cite{NeedellT08}, OMPR \cite{JainTD11} and SP \cite{DaiM09}. We compared them with a standard implementation of the L1 projected scaled sub-gradient technique \cite{Schmidt10} for the lasso problem and a greedy method FoBa \cite{Zhang11} for the same.\\

\textbf{Evaluation Metrics}: For the baseline noise level $\sigma = 0.1$, we found that all the algorithms were able to recover the support set within an error of $2\%$. Consequently, our focus shifted to running times for these experiments. In the experiments where noise levels were varied, we recorded, for each method, the number of undiscovered support set elements.\\

\textbf{Results}: Figure\ref{fig:expts} describes the results of our experiments in graphical form. For sake of clarity we included only HTP, GraDeS, L1 and FoBa results in these graphs. Graphs for the other algorithms CoSaMP, SP and OMPR can be seen in the supplementary material. The graphs indicate that whereas hard thresholding techniques are equally effective as L1 and greedy techniques for recovery in noisy settings, as indicated by Figure\ref{subfig:noise}, the former can be much more efficient and scalable than the latter. For instance, as Figure\ref{subfig:dimen}, for the base level of $p = 20000$, HTP was $150\times$ faster than the L1 method. For higher values of $p$, the runtime gap widened to more than $350\times$. We also note that in both these cases, HTP actually offered exact support recovery whereas L1 was unable to recover $2$ and $4$ support elements respectively.

Although FoBa was faster than L1 on Figure\ref{subfig:dimen} experiments, it was still slower than HTP by $50\times$ and $90\times$ for $p = 20000$ and $25000$ respectively. Moreover, due to its greedy and incremental nature, FoBa was found to suffer badly in settings with larger true sparsity levels. As Figure~\ref{subfig:sparsity} indicates, for even moderate sparsity levels of $\so = 300$ and $500$, FoBa is $60-75\times$ slower than HTP. As mentioned before, the reason for this slowdown is the greedy approach followed by FoBa: whereas HTP took less than $5$ iterations to converge for these two problems, FoBa spend $300$ and $500$ iterations respectively. GraDeS was found to offer much lesser run times in comparison being slower than HTP by $30-40\times$ for larger values of $p$ and $2-5\times$ slower for larger values of $\so$.\\

\textbf{Experiments on badly conditioned problems.} We also ran experiments to verify the performance of IHT algorithms in high condition number setting. Values of $p, \so$ and $\sigma$ were kept at baseline levels. After selecting the optimal parameter vector $\btbar$, we selected $\so/2$ random coordinates from its support and $\so/2$ random coordinates outside its support and constructed a covariance matrix with heavy correlations between these chosen coordinates. The condition number of the resulting matrix was close to $50$. Samples were drawn from this distribution and the recovery properties of the different IHT-style algorithms was observed as the projected sparsity levels $\s$ were increased. Our results (see Figure~\ref{subfig:oversamp}) corroborate our theoretical observation that these algorithms show a remarkable improvement in recovery properties for ill-conditioned problems with an enlarged projection size.

%All methods were found to be rather stable with respect to variation in sample size. This is partly due to the fact that on the one hand, increased sample size simplifies the recovery problem, on the other hand it makes individual iterations more expensive to compute.
%%% Local Variables: 
%%% mode: latex
%%% TeX-master: "lowsparse_matrix"
%%% End: 

%\vspace{-4pt}
\section{Discussion and Conclusions}
\label{sec:conc}
%\vspace{-4pt}
In our work we studied iterative hard thresholding algorithms and showed that these techniques can offer global convergence guarantees for arbitrary, possibly non-convex, differentiable objective functions, which nevertheless satisfy Restricted Strong Convexity/Smoothness (RSC/RSM) conditions. Our results apply to a large family of algorithms that includes existing algorithms such as IHT, GraDeS, CoSaMP, SP and OMPR. Previously the analyses of these algorithms required stringent RIP conditions that did not allow the (restricted) condition number to be larger than universal constants specific to these algorithms.

Our basic insight was to relax this stringent requirement by running these iterative algorithms with an enlarged support size. We showed that guarantees for high-dimensional M-estimation follow seamlessly from our results by invoking results on RSC/RSM conditions that have already been established in the literature for a variety of statistical settings. Our theoretical results put hard thresholding methods on par with those based on convex relaxation or greedy algorithms. Our experimental results demonstrate that hard thresholding methods outperform convex relaxation and greedy methods in terms of running time, sometime by orders of magnitude, all the while offering competitive or better recovery properties.

Our results apply to sparsity and low rank structure, arguably two of the most commonly used structures in high dimensional statistical learning problems. In future work, it would be interesting to generalize our algorithms and their analyses to more general structures. A unified analysis for general structures will probably create interesting connections with existing unified frameworks such as those based on decomposability \cite{NegahbanRWY2012} and atomic norms \cite{ChandrasekaranRPW2012}.

%\clearpage
%{\small
\bibliography{refs}
\bibliographystyle{unsrt}
%}
%\clearpage
\appendix
\section{Proofs for Section~\ref{sec:algo}}\label{app:iht}

\newcommand{\btz}{\bt^{\z}}

\begin{proof}[Proof of Lemma~\ref{lem:ht1}]
Without loss of generality, assume that we have reordered coordinates such that $|\z_1| \geq |\z_2| \geq \ldots \geq |\z_I|$. Since the projection operator $P_s(\cdot)$ operates by selecting the largest elements by magnitude, we have $\bt_1 = \z_1, \ldots, \bt_s = \z_s$ and $\bt_{s+1} = \bt_{s+2} = \ldots = \bt_{|I|} = 0$.

Also define $\btz = P_{\so}(\z)$. By the above argument, we have $\btz_1 = \z_1, \ldots, \btz_{\so} = \z_{\so}$ and $\btz_{\so+1} = \btz_{\so+2} = \ldots = \btz_{|I|} = 0$. Now we have
\begin{align}
\frac{\|\btz - \z\|}{|I|-\so} - \frac{\|\bt - \z\|}{|I|-s} &= \frac{1}{|I|-\so}\sum_{i = \so+1}^{s}\z_i^2 + \left(\frac{1}{|I|-\so} - \frac{1}{|I|-s}\right)\sum_{i = s+1}^{|I|}\z_i^2\nonumber\\
& \geq \frac{s-\so}{|I|-\so}\z_s^2 + \frac{\so-s}{(|I|-\so)(|I|-s)}(|I|-s)\z_{s+1}^2 \geq 0,
\end{align}
since the coordinates of $\z$ are arranged in decreasing order of magnitude. Combining the above with the observation that, due to the projection property $\|\bto - \z\| \geq \|\btz - \z\|$, proves the result.
\end{proof}

\begin{proof}[Proof of Theorem~\ref{thm:iht_simp}]
%  Our proof is inspired by the proof of the GradeS method by \cite{GargK08}, but is for general function $f$ and holds under general RSC condition. 
Recall that $\ttn=P_s(\btt-\frac{\eta'}{\ss}\gt)$ where $\eta' = \tfrac{2}{3} < 1$. Let $\St=supp(\btt)$, $\So=supp(\bto)$, and $\Stn=supp(\ttn)$. Also, let $\It=\So\cup \St\cup \Stn$. %Note that, $\btt, \ttn, \bto$ all have support restricted to $\It$. 

Now, using the RSS property and the fact that $supp(\btt)\subseteq \It$ and $supp(\ttn)\subseteq \It$, we have: 
\begin{align}
  f(\ttn)-f(\btt)&\leq \ip{\ttn-\btt}{\gt}+\frac{\ss}{2}\|\ttn-\btt\|_2^2,\nonumber\\
&= \frac{\ss}{2}\|\ttn_{\It}-\btt_{\It}+\frac{\eta'}{\ss} \cdot \gt_{\It}\|_2^2-\frac{(\eta')^2}{2 \ss}\|\gt_{\It}\|_2^2+(1-\eta')\ip{\ttn-\btt}{\gt}.\label{eq:iht_n1_app}
\end{align}
As $supp(\btt)=\St$, $supp(\ttn)=\Stn$ and $\St\backslash \Stn, \Stn$ are disjoint, we have: 
\begin{align}
\ip{\ttn-\btt}{\gt}&= -\ip{\btt_{\St\backslash \Stn}}{\gt_{\St\backslash \Stn}}+\ip{\ttn_{\Stn}-\btt_{\Stn}}{\gt_{\Stn}},\nonumber\\
&\stackrel{\zeta_1}{=}-\ip{\btt_{\St\backslash \Stn}}{\gt_{\St\backslash \Stn}}-\frac{\eta'}{\ss}\|\gt_{\Stn}\|_2^2,\nonumber\\
&\stackrel{\zeta_2}{\leq}\frac{\eta'}{2\ss}\|\gt_{\Stn\backslash\St}\|_2^2-\frac{\eta'}{2\ss}\|\gt_{\St\backslash\Stn}\|_2^2-\frac{\eta'}{\ss}\|\gt_{\Stn}\|_2^2,\nonumber\\
&\stackrel{\zeta_3}{=} -\frac{\eta'}{2\ss}\|\gt_{\Stn\backslash\St}\|_2^2-\frac{\eta'}{2\ss}\|\gt_{\St\backslash\Stn}\|_2^2 - \frac{\eta'}{\ss}\|\gt_{\St \cap \Stn}\|_2^2 \nonumber\\
&\leq -\frac{\eta'}{2\ss}\|\gt_{\St\cup\Stn}\|_2^2, \label{eq:iht_n2_app}
%&=\frac{\ss}{2}\|\ttn_{\It}-\btt_{\It}+\frac{\eta}{\ss} \cdot \gt_{\It}\|_2^2-\frac{\eta^2}{2 \ss}\|\gt_{\It}\|_2^2-\frac{(1-\eta)\eta}{2\ss}\|\gt_{\Stn\backslash\St}\|_2^2-\frac{(1-\eta)\eta}{2\ss}\|\gt_{\St\backslash\Stn}\|_2^2-\nonumber\\&\qquad\qquad\frac{\eta(1-\eta)}{\ss}\|\gt_{\Stn\cap \St}\|_2^2,\nonumber\\
%&=\frac{\ss}{2}\|\ttn_{\It}-\btt_{\It}+\frac{\eta}{\ss} \cdot \gt_{\It}\|_2^2-\frac{\eta^2}{2 \ss}\|\gt_{\It\backslash(\St\cup \So)}\|_2^2-\frac{\eta^2}{2 \ss}\|\gt_{\St\cup \So}\|_2^2-\nonumber\\&\qquad\qquad\frac{(1-\eta)\eta}{2\ss}\|\gt_{\Stn\backslash\St}\|_2^2-\frac{(1-\eta)\eta}{2\ss}\|\gt_{\St\backslash\Stn}\|_2^2-\frac{\eta(1-\eta)}{\ss}\|\gt_{\Stn\cap \St}\|_2^2,\label{eq:iht_2_app}
%&\stackrel{\zeta_1}{\leq}\frac{\ss}{2}\|\ttn_{\It}-\btt_{\It}+\frac{1}{2\ss} \cdot \gt_{\It}\|_2^2-\frac{1}{8 \ss}\|\gt_{\It}\|_2^2+\frac{1}{2}\ip{\ttn}{\gt}+\left(f(0)-f(\btt)\right),\nonumber\\
%&\stackrel{\zeta_1}{\leq} \frac{\ss}{2}\cdot \frac{|\It|-\s}{|\It|-\so}\cdot\|\bto_{\It}-\btt_{\It}+\frac{1}{\ss} \cdot \gt_{\It}\|_2^2-\frac{1}{2 \ss}(\|\gt_{\It\backslash (\St \cup \So)}\|_2^2+\|\gt_{\St\cup \So}\|_2^2), \label{eq:iht_2_app}
\end{align}
where the equality $\zeta_1$ follows from the gradient step, i.e., $\ttn_{\Stn}=\btt_{\Stn}-\frac{\eta'}{\ss}\gt_{\Stn}$. The inequality $\zeta_2$ follows using the fact that $\ttn$ is obtained using hard thresholding and the fact that $|\St\backslash \Stn|=|\Stn\backslash\St|$, as follows:
\begin{align}
  \|\btt_{\St\backslash \Stn}-\frac{\eta'}{L}\gt_{\St\backslash \Stn}\|_2^2\leq\|\ttn_{\Stn\backslash\St}\|_2^2=\frac{(\eta')^2}{\ss^2}\|\gt_{\Stn\backslash\St}\|_2^2.
\end{align}
%hard-thresholding and the fact that $|\St\backslash \Stn|=|\Stn\backslash\St|$. 
%where $\zeta_1$ follows from Lemma~\ref{lem:ht1} and using $\It=(\It\backslash (\St\cup \So))\cup (\St\cup \So)$. 
The equality $\zeta_3$ follows from $\|\gt_{\Stn}\|_2^2 = \|\gt_{\Stn\backslash\St}\|_2^2 + \|\gt_{\St \cap \Stn}\|_2^2$. 

Hence, using \eqref{eq:iht_n1_app} and \eqref{eq:iht_n2_app}, we have: 
\begin{align}
  f(\ttn)-f(\btt)&\leq  \frac{\ss}{2}\|\ttn_{\It}-\btt_{\It}+\frac{\eta'}{\ss} \cdot \gt_{\It}\|_2^2-\frac{(\eta')^2}{2 \ss}\|\gt_{\It}\|_2^2-\frac{\eta'(1-\eta')}{2\ss}\|\gt_{\St\cup\Stn}\|_2^2,\nonumber\\
&=\frac{\ss}{2}\|\ttn_{\It}-\btt_{\It}+\frac{\eta'}{\ss} \cdot \gt_{\It}\|_2^2-\frac{(\eta')^2}{2 \ss}\|\gt_{\It\backslash(\St\cup \So)}\|_2^2-\frac{(\eta')^2}{2 \ss}\|\gt_{\St\cup \So}\|_2^2 \nonumber\\
&\quad -\frac{\eta'(1-\eta')}{2\ss}\|\gt_{\St\cup\Stn}\|_2^2.\label{eq:iht_2_app}
\end{align}
 
Next, let us try to upper bound the first two terms on the right hand side above.
Since $\It\backslash (\St \cup \So)=\Stn\backslash (\St\cup \So) \subseteq \Stn$, we have $\ttn_{\It\backslash (\St \cup \So)}=\btt_{\It\backslash (\St \cup \So)}-\frac{\eta'}{\ss}\gt_{\It\backslash (\St \cup \So)}$. However, as $\btt_{\It\backslash \St}=0$, we actually have $\ttn_{\It\backslash (\St \cup \So)}=-\frac{\eta'}{\ss}\gt_{\It\backslash (\St \cup \So)}$.
Now let us choose a set $R \subseteq \St\backslash \Stn$ such that $|R|=|\Stn\backslash (\St\cup \So)|$.
Such a choice is possible since $|\Stn\backslash (\St\cup \So)|=|\St\backslash \Stn|-|(\Stn\cap \So)\backslash\St|$ (which itself is a consequence of the fact that $|\Stn|=|\St|$).
Moreover, since $\ttn$ is obtained by hard-thresholding $\left(\btt-\frac{\eta'}{\ss}\gt\right)$, for any choice of $R$ made above, we have:
\begin{align}
  \frac{(\eta')^2}{\ss^2}\|\gt_{\It\backslash (\St \cup \So)}\|_2^2=\|\ttn_{\It\backslash (\St \cup \So)}\|_2^2\geq \|\btt_R-\frac{\eta'}{\ss}\gt_R\|_2^2.\label{eq:iht_5}
\end{align}
Using above equation, and the fact that $\ttn_{R}=0$ (since $R \subseteq \overline{\Stn}$), we have: 
\begin{align}
&\quad \frac{\ss}{2}\|\ttn_{\It}-\btt_{\It}+\frac{\eta'}{\ss} \cdot \gt_{\It}\|_2^2-\frac{(\eta')^2}{2 \ss}\|\gt_{\It\backslash (\St \cup \So)}\|_2^2 \nonumber \\
&\leq \frac{\ss}{2}\|\ttn_{\It}-\btt_{\It}+\frac{\eta'}{\ss} \cdot \gt_{\It}\|_2^2-\frac{\ss}{2}\|\ttn_R - \btt_R + \frac{\eta'}{\ss}\gt_R\|_2^2 \nonumber \\
&= \frac{\ss}{2}\|\ttn_{\It\backslash R}-\btt_{\It\backslash R}+\frac{\eta'}{\ss} \cdot \gt_{\It\backslash R}\|_2^2.\label{eq:iht_6}
\end{align}
We can bound the size of $\It\backslash R$ as $|\It\backslash R|\leq |\Stn|+|(\St\backslash \Stn)\backslash R|+|\So|\leq \s+|(\Stn\cap \So)\backslash\St|+\so\leq \s+2\so.$ Also, since $\Stn\subseteq (\It\backslash R)$, we have $\ttn_{\It\backslash R}=P_{\s}(\btt_{\It\backslash R}-\frac{\eta'}{\ss}\gt_{\It\backslash R})$.

Using the above observation with \eqref{eq:iht_6} and Lemma~\ref{lem:ht1}, we get: 
\begin{align}
  & \quad \frac{\ss}{2}\|\ttn_{\It}-\btt_{\It}+\frac{\eta'}{\ss} \cdot \gt_{\It}\|_2^2-\frac{(\eta')^2}{2 \ss}\|\gt_{\It\backslash (\St \cup \So)}\|_2^2 \nonumber\\
  &\leq \frac{\ss}{2}\cdot \frac{|\It\backslash R|-\s}{|\It\backslash R|-\so}\|\bto_{\It\backslash R}-\btt_{\It\backslash R}+\frac{\eta'}{\ss} \cdot \gt_{\It\backslash R}\|_2^2,\nonumber\\
&\stackrel{\zeta_1}{\leq} \frac{\ss}{2}\cdot \frac{2\so}{\s+\so}\|\bto_{\It}-\btt_{\It}+\frac{\eta'}{\ss} \cdot \gt_{\It}\|_2^2,\nonumber\\
&= \frac{2\so}{\s+\so}\cdot \left(\eta'\ip{\bto-\btt}{\gt}+\frac{\ss}{2}\|\bto-\btt\|_2^2+\frac{(\eta')^2}{2\ss}\|\gt_{\It}\|_2^2\right),\nonumber\\
&\stackrel{\zeta_2}{\leq} \frac{2\so}{\s+\so}\cdot \left(\eta' f(\bto)-\eta' f(\btt)+\frac{\ss-\eta'\sc}{2}\|\bto-\btt\|_2^2+\frac{(\eta')^2}{2\ss}\|\gt_{\It}\|_2^2\right),\label{eq:iht_7}
\end{align}
where the inequality $\zeta_1$ follows by $|\It\backslash R|\leq \s+2\so$ as shown earlier and the observation that $\frac{x-a}{x-b}$ is a positive and increasing function on the interval $x\geq a$ if $a \geq b \geq 0$. Note that since we have $\Stn\subseteq (\It\backslash R)$, we get $|\It\backslash R| \geq s$.
The inequality $\zeta_2$ follows by using RSC.%, and Lemma~\ref{lem:diff}. $\zeta_3$ follows by using the assumption that $f(\btt)\geq \beta f(\bto)$. 

Using \eqref{eq:iht_2_app},  \eqref{eq:iht_7}, and using $\Stn\backslash (\St\cup \So)\subseteq (\Stn\cup \St)$, we get: 
\begin{align}
\nonumber f(\ttn)-f(\btt) \leq {} & \frac{2\so}{\s+\so}\cdot \left(\eta' f(\bto)-\eta'	 f(\btt)+\frac{\ss-\eta'\sc}{2}\|\bto-\btt\|_2^2+\frac{(\eta')^2}{2\ss}\|\gt_{\It}\|_2^2\right)\\
& -\frac{(\eta')^2}{2\ss}\|\gt_{\St\cup \So}\|_2^2 - \frac{\eta'(1-\eta')}{2\ss}\|\gt_{\Stn\backslash(\St\cup \So)}\|_2^2. \label{eq:iht_8}
\end{align}

We now set $\eta' = 2/3$ as per our earlier choice and set $s = 32\left(\frac{L}{\alpha}\right)^2\so$, so that we have $\frac{2\so}{\s+\so}\leq \frac{\sc^2}{16\ss(\ss - \eta'\sc)}$. Since $L \geq \alpha$, we also have $\frac{\sc^2}{16\ss(\ss - \eta'\sc)} \leq \frac{3}{16}$. Using these inequalities, we now rearrange the terms in \eqref{eq:iht_8} above.

\begin{align}
\nonumber f(\ttn)-f(\btt) \leq {} & \frac{2\so}{\s+\so}\cdot \eta'\cdot \left(f(\bto) - f(\btt)\right)+\frac{\sc^2}{32\ss}\|\bto-\btt\|_2^2+\frac{1}{24\ss}\|\gt_{\It}\|_2^2\\
& -\frac{2}{9\ss}\|\gt_{\St\cup \So}\|_2^2 - \frac{1}{9\ss}\|\gt_{\Stn\backslash(\St\cup \So)}\|_2^2. \label{eq:iht_8_5}
\end{align}

Splitting $\|\gt_{\It}\|_2^2 = \|\gt_{\St\cup \So}\|_2^2 + \|\gt_{\Stn\backslash(\St\cup \So)}\|_2^2$ gives us

\begin{align}
f(\ttn)-f(\btt) \leq {} & \frac{2\so}{\s+\so}\cdot \eta'\cdot\left(f(\bto)-f(\btt)\right)-\frac{1}{2\ss}\left(\frac{13}{36}\|\gt_{\St\cup \So}\|_2^2-\frac{\sc^2}{16}\|\bto-\btt\|_2^2\right)\nonumber\\
& - \frac{1}{2\ss}\cdot\left(\frac{4}{9}-\frac{1}{12}\right)\|\gt_{\Stn\backslash(\St\cup \So)}\|_2^2,\nonumber\\
\leq {} & \frac{2\so}{\s+\so}\cdot \eta'\cdot\left(f(\bto)-f(\btt)\right)-\frac{13}{72\ss}\left(\|\gt_{\St\cup \So}\|_2^2-\frac{\sc^2}{4}\|\bto-\btt\|_2^2\right)\nonumber\\
\leq {} & \frac{2\so}{\s+\so}\cdot \eta'\cdot \left(f(\bto)-f(\btt)\right)-\frac{\sc}{12\ss} \left(f(\btt)-f(\bto)\right),\label{eq:iht_9}
\end{align}
where the last inequality above follows using Lemma~\ref{lem:gt_sc_supp}. The result now follows by observing that $\frac{2\so}{\s+\so}\geq 0$. 
% That is, there exists a set $R\subseteq \St\backslash \Stn$, s.t., $R$ consists of  \begin{align}
% &\stackrel{\zeta_1}{\leq} \frac{|\It|-\s}{|\It|-\so} \cdot \frac{\ss}{2}\cdot \|\bto_{\It}-\btt_{\It}+\frac{1}{\ss} \cdot \gt_{\It}\|_2^2-\frac{1}{2 \ss}\|\gt_{\It}\|_2^2,\nonumber\\
% &\stackrel{\zeta_2}{\leq} \frac{|\It|-\s}{|\It|-\so}\cdot \left(\ip{\bto-\btt}{\gt}+\frac{\sc}{2}\|\bto-\btt\|_2^2+\frac{\ss-\sc}{2}\|\bto-\btt\|_2^2\right),\nonumber\\
% &\stackrel{\zeta_3}{\leq} \frac{|\It|-\s}{|\It|-\so}\cdot \left(f(\bto)-f(\btt)+\frac{L-\alpha}{\alpha}(\sqrt{f(\btt)}+\sqrt{f(\bto)})^2\right)\nonumber\\
% &\stackrel{\zeta_4}{\leq} \frac{|\It|-\s}{|\It|-\so}\cdot \left(\frac{1}{\beta}-1+\frac{L-\alpha}{\alpha}\left(1+\frac{1}{\sqrt{\beta}}\right)^2\right)\cdot f(\btt)
% \end{align}
\end{proof}

\begin{lemma}\label{lem:gt_sc_supp}
  $$\left(\|\gt_{\St\cup \So}\|_2^2-\frac{\sc^2}{4}\|\bto-\btt\|_2^2\right)\geq \frac{\sc}{2}\cdot \left(f(\btt)-f(\bto)\right).$$
\end{lemma}
\begin{proof}
  Using the RSC property, we have: 
\begin{align}
  f(\btt)-f(\bto)&\leq \ip{\gt}{\btt-\bto}-\frac{\sc}{2}\|\bto-\btt\|_2^2\nonumber\\
  &=\ip{\gt_{\St\cup \So}}{\btt_{\St\cup \So}-\bto_{\St\cup \So}}-\frac{\sc}{2}\|\bto-\btt\|_2^2,\nonumber\\
&\leq \|\gt_{\St\cup \So}\|_2\|{\btt-\bto}\|_2-\frac{\sc}{2}\|\bto-\btt\|_2^2.\label{eq:iht_3}
\end{align}
Now, 
\begin{align}
  \|\gt_{\St\cup \So}\|_2^2-\frac{\sc^2}{4}\|\bto-\btt\|_2^2&=\left(\|\gt_{\St\cup \So}\|_2-\frac{\sc}{2}\|\bto-\btt\|_2\right)\left(\|\gt_{\St\cup \So}\|_2+\frac{\sc}{2}\|\bto-\btt\|_2\right),\nonumber\\
&\geq \frac{\left(f(\btt)-f(\bto)\right)}{\|\btt-\bto\|_2}\cdot \left(\|\gt_{\St\cup \So}\|_2+\frac{\sc}{2}\|\bto-\btt\|_2\right) \nonumber \\
& \geq \frac{\sc}{2}\cdot \left(f(\btt)-f(\bto)\right),
\end{align}
where the first inequality above follows from \eqref{eq:iht_3}. 
\end{proof}

%%% Local Variables: 
%%% mode: latex
%%% TeX-master: "rsc_iht"
%%% End: 

\section{Proofs for Section~\ref{sec:stat}}
\label{app:stat}

\begin{proof}[Proof of Theorem~\ref{thm:genstat}]
Let $\bto$ be the empirical loss minimizer over the set of $\s$-sparse vectors. Then invoking Theorem~\ref{thm:iht_simp} with $f = \loss(\cdot;Z_{1:n})$, we get
\begin{align*}
\loss(\bt^\tau,Z_{1:n}) -\epsilon &\le \loss(\bto,Z_{1:n}) \le \loss(\btbar,Z_{1:n}) \\
&\le \loss(\bt^{\tau};Z_{1:n}) + \ip{\nabla \loss(\btbar;Z_{1:n})}{(\btbar-\bt^{\tau})} - \frac{\sc_{\s+\so}}{2} \| \btbar - \bt^\tau \|_2^2 
\end{align*}
where the 2nd inequality is by definition of $\bto$ and 3rd is by RSC (since $\bto,\bt^\tau$ are $\so,\s$ sparse). Duality gives us the upper bound
\[
\ip{\nabla \loss(\btbar;Z_{1:n})}{(\btbar-\bt^{\tau})} \le \| \nabla \loss(\btbar;Z_{1:n}) \|_\infty \|  \btbar-\bt^{\tau} \|_1
\le \sqrt{\s+\so} \| \nabla \loss(\btbar;Z_{1:n}) \|_\infty \|  \btbar-\bt^{\tau} \|_2
\]
Combining the last two inequalities and rearranging gives a quadratic inequality in $\|\btbar - \bt^\tau\|_2$:
\[
\frac{\sc_{\s+\so}}{2} \| \btbar - \bt^\tau \|_2^2 - \sqrt{\s+\so} \| \nabla \loss(\btbar;Z_{1:n}) \|_\infty \|  \btbar-\bt^{\tau} \|_2 - \epsilon \le 0
\] 
that immediately yields the result.
\end{proof}

\section{Proofs for Section~\ref{sec:fully}}\label{app:fully}
\newcommand{\Hs}{\bm{H}}
\newcommand{\melem}{MD_t}
\newcommand{\felem}{FA_t}
\newcommand{\bro}[1]{\left({#1}\right)}
\begin{proof}[Proof of Lemma~\ref{lem:full_1}]
We will start by proving a more general result of which the claimed result will be a corollary. More specifically, we shall prove that for any $\gamma\geq \frac{1}{\sc}$, we have
$$2\gamma(f(\btt)-f(\bto))\leq 2\gamma\left(f(\btt)-f(\bto)+\frac{\sc}{2}\cdot\left(1-\frac{1}{\sc\gamma}\right)\|\btt-\bto\|_2^2\right)\leq \gamma^2\|\gt_{\St\cup \So}\|_2^2-\|\btt_{\St\backslash \So}\|_2^2,$$
Setting $\gamma = \frac{1}{\sc}$ will yield the claimed result. It is easy to see that the following inequality holds trivially since $\gamma \geq \frac{1}{\sc}$
\[
2\gamma(f(\btt)-f(\bto))\leq 2\gamma\left(f(\btt)-f(\bto)+\frac{\sc}{2}\cdot\left(1-\frac{1}{\sc\gamma}\right)\|\btt-\bto\|_2^2\right).
\]
For the second inequality, we first use the RSC condition to obtain: %perform a Taylor series expansion for $f$ around $\btt$
\[
f(\bto) - f(\btt) \geq \ip{\bto - \btt}{\gt} + \frac{\sc}{2}\|\btt-\bto\|_2^2.
\]
%where $\Hs_\xi$ is the Hessian of $f$ at a point $\xi = \lambda\btt + (1-\lambda)\bto$ for some $0 \leq \lambda \leq 1$. 
Now let $\melem = \So\backslash\St$ be the set of true support elements missing from $\btt$ and $\felem = \St\backslash\So$ be the set of incorrect elements included in the support of $\btt$. Since $\btt$ is obtained by a ``fully corrective'' process (recall $\btt=\arg\min_{\bt, supp(\bt)\subseteq \St}f(\bt)$), we have $\gt_{\St} = \bm{0}$. Thus $\ip{\bto - \btt}{\gt} = \ip{\bto_{\melem}}{\gt_{\melem}}$.

Putting this into the above expansion gives
\begin{align}
	\label{eq:full-1-1}
	\ip{\bto_{\melem}}{\gt_{\melem}} \leq f(\bto) - f(\btt) - \frac{\sc}{2}\|\btt-\bto\|_2^2
\end{align}
We now present some simple inequalities that will help us get our desired bounds. Firstly, we have
\begin{align}
	\label{eq:full-1-2}
	\|\bto_{\melem} + \gamma\gt_{\melem}\|_2^2 = \|\bto_{\melem}\|_2^2 + \gamma^2\|\gt_{\melem}\|_2^2 + 2\gamma\ip{\bto_{\melem}}{\gt_{\melem}} \geq 0,
\end{align}
since the first expression is a norm. Next, since $\melem \cap \felem = \emptyset$, we have
\begin{align}
	\label{eq:full-1-3}
	\|\bto - \btt\|_2^2 \geq \|\bto_{\melem}\|_2^2 + \|\btt_{\felem}\|_2^2.
\end{align}
%Thirdly, by RSC, and the fact that $\|\xi\|_0 \leq \s + \so$, we have $\Hs_\xi \succeq (1 - \sc_{\s+\so}(f)) I \succeq (1 - \sc) I$ by monotonicity of the RSC parameter. This gives us
%\begin{align}
%	\label{eq:full-1-4}
%	(1-\alpha)\|\bto - \btt\|_2^2 \leq \|\Hs^{1/2}_\xi(\btt-\bto)\|_2^2.
%\end{align}
Putting equations \ref{eq:full-1-1} and \ref{eq:full-1-2}, we have: % together and some rearrangement gives us
\begin{equation}
  \label{eq:full-1-4}
  2\gamma\left(f(\btt) - f(\bto) + \frac{\sc}{2}\|\btt-\bto\|_2^2\right)\leq \|\bto_{\melem}\|_2^2 + \gamma^2\|\gt_{\melem}\|_2^2.
\end{equation}
Now, using \eqref{eq:full-1-3}, we get:
\[
2\gamma\bro{f(\btt) - f(\bto) + \frac{\sc}{2}\bro{1-\frac{1}{\sc\gamma}}\|\btt-\bto\|_2^2} \leq \gamma^2\|\gt_{\melem}\|_2^2 - \|\btt_{\felem}\|_2^2
\]
We finish off the proof by noticing that since $\gt_{\St} = \bm{0}$, we have $\|\gt_{\melem}\|_2^2 = \|\gt_{\St\cup \So}\|_2^2$
\end{proof}

\begin{proof}[Proof of Theorem~\ref{thm:tstage}]
Let $\z^t_{\St}=\btt_{\St}$, $\z^t_{Z^t\backslash \St}=-\frac{1}{\ss}\gt_{Z^t\backslash \St}$, and $\z^t_{\overline{Z^t}}=0$. 

Then, using the RSS property, we have: 
\begin{align}
  f(\z^t)-f(\btt)&\leq \ip{\z^t-\btt}{\gt}+\frac{\ss}{2}\|\z^t-\btt\|_2^2,\nonumber\\
&\stackrel{\zeta_1}{\leq} -\frac{1}{\ss}\|\gt_{Z^t\backslash \St}\|_2^2+\frac{\ss}{2}\|\z^t_{Z^t\backslash \St}\|_2^2,\nonumber\\
&\stackrel{\zeta_2}{=} -\frac{1}{2\ss}\cdot \|\gt_{Z^t\backslash \St}\|_2^2,\nonumber\\
&\stackrel{\zeta_3}{\leq}-\frac{1}{2\ss}\cdot \|\gt_{\So\backslash \St}\|_2^2,\nonumber\\
&\stackrel{\zeta_4}{\leq}-\frac{\sc}{\ss}\cdot\left(f(\btt)-f(\bto)\right),
\label{eq:tst_0}
\end{align}
where $\zeta_1$ follows by observing $\gt_{\St}=0$, and $\St\subseteq Z^t$. $\zeta_2$ follows by $\z^t_{Z^t\backslash \St}=-\frac{1}{\ss}\gt_{Z^t\backslash \St}$. $\zeta_3$ follows by $\ell\geq \so$, and $Z^t\backslash \St$ are the $\ell$ largest elements of $|\gt_{Z^t\backslash \St}|$. 

Now, using Lemma~\ref{lem:ht2} and \eqref{eq:tst_0} along with $f(\ttn)\leq f(\widetilde{\bt}^t)$ and $f(\bbt)\leq f(\z^t)$, we have: 
\begin{equation}
  \label{eq:tst_1}
  f(\ttn)-f(\bto)\leq \left(1-\frac{\sc}{\ss}\right)\cdot\left(1+\frac{\ss}{\sc}\cdot \frac{\ell }{\s+\ell-\so } \right)\cdot \left(f(\btt)-f(\bto)\right).
\end{equation}
Theorem now follows by using the above equation with the assumption that $\s+\ell-\so\geq \frac{4\ss^2 \cdot \ell}{\sc^2}$. 
\end{proof}

%\section{Proofs for Section~\ref{sec:pht}}
%\label{app:pht}
\begin{proof}[Proof of Theorem~\ref{thm:pht}]
Using RSS property: 
\begin{align}
  f(\vt)-f(\btt)&\leq \ip{\vt-\btt}{\gt}+\frac{\ss}{2}\|\vt-\btt\|_2^2,\nonumber\\
&\stackrel{\zeta_1}{\leq} -\eta\|\gt_{\Stn\backslash \St}\|_2^2+\frac{\ss}{2}(\|\vt_{\Stn\backslash \St}\|_2^2+\|\btt_{\St\backslash \Stn}\|^2),\nonumber\\
&\stackrel{\zeta_2}{\leq} -\eta\|\gt_{\Stn\backslash \St}\|_2^2+\ss \|\vt_{\Stn\backslash \St}\|_2^2,\nonumber\\
&\stackrel{\zeta_3}{=} -\left(1-\eta\cdot \ss\right)\cdot \eta\cdot \|\gt_{\Stn\backslash \St}\|_2^2,\label{eq:pht_0}
\end{align}
where $\zeta_1$ follows by observing that $\gt_{\St}=0$ and $\vt_{\Stn\backslash \St}=-\eta \gt_{\Stn\backslash \St}$. $\zeta_2$ follows by the property of PHT operator which ensures that each element of $\vt_{\Stn\backslash \St}$ is bigger than $\btt_{\St\backslash \Stn}$ and by using $|\Stn\backslash \St|=|\St\backslash \Stn|$. $\zeta_3$ follows by using $\vt_{\Stn\backslash \St}=-\eta \gt_{\Stn\backslash \St}$. 

Similar to the analysis given in \cite{JainTD11}, we divide the analysis in three mutually exclusive cases. The lemma then follows by combining \eqref{eq:pht_0} with the case-by-case analyses below and observing that $f(\ttn) \leq f(\vt)$ because of the fully corrective step.

%Now,  there are two cases when $|\Stn\backslash \St|<|\So\backslash \St|$:

{\bf Case 1}: $|\Stn\backslash \St|=\ell<|\So\backslash \St|$. In this case, 
As  $\vt_{\Stn\backslash \St}$ is obtained by selecting $|\Stn\backslash \St|$ largest elements of $\z^t_{\overline{\St}\cup bot^t}$. Hence, 
\begin{eqnarray}
\nonumber
  \eta^2\|\gt_{\Stn\backslash \St}\|_2^2 &\geq& \min\left(1, \frac{|\Stn\backslash \St|}{|\So\backslash \St|}\right)\|\z^t_{\So\backslash \St}\|_2^2\\\nonumber
										 &=& \eta^2\min\left(1, \frac{|\Stn\backslash \St|}{|\So\backslash \St|}\right)\|\gt_{\So\backslash \St}\|_2^2\\
  \label{eq:pht_6}						 &=& \eta^2\min\left(1, \frac{|\Stn\backslash \St|}{|\So\backslash \St|}\right)\|\gt_{\So \cup \St}\|_2^2,
\end{eqnarray}
since $\gt_{\St} = 0$. Now, using the fact that $|{\Stn\backslash \St}|=\ell$, ${|\So\backslash \St|}\leq \so$, and using Lemma~\ref{lem:full_1}, we have: 
\begin{equation}
\|\gt_{\Stn\backslash \St}\|_2^2\geq  2\sc\cdot \min\left(1,\frac{\ell}{\so}\right)\left(f(\btt)-f(\bto)\right). \label{eq:pht_9}
\end{equation}
{\bf Case 2}: $|\Stn\backslash \St|<\ell$, $|\Stn\backslash \St|\leq |\So\backslash \St|$. In this case, each element of $\z^t_{\Stn\cap \St}$ is larger than  each element of $\z^t_{\So\backslash (\Stn\cup \St)}$, else that element of ${\So\backslash (\Stn\cup \St)}$ would have been selected. That is, 
$$\frac{\|\z^t_{\So\backslash (\Stn\cup \St)}\|_2^2}{|\So\backslash (\Stn\cup \St)|}\leq \frac{\|z^t_{(\Stn\cap \St)\backslash \So}\|_2^2}{|{(\Stn\cap \St)\backslash \So}|}.$$
Using $\z^t_{\So\backslash \St}=-\eta \gt_{\So\backslash \St}$, $\z^t_{\St}=\btt_{\St}$ and $(\Stn\cap \St)\backslash \So \subseteq \St\backslash \So$, we have: 
\begin{equation}
  \label{eq:pht_3}
  \eta^2\|\gt_{\So\backslash (\St\cup \Stn)}\|_2^2\leq \frac{\so}{\s-\ell-\so}\|\btt_{\St\backslash \So}\|_2^2,
\end{equation}
where the bound on $\frac{|\So\backslash (\Stn\cup \St)|}{|{(\Stn\cap \St)\backslash \So}|}$ follows by observing ${|\So\backslash (\Stn\cup \St)|}\leq \so$ and ${|{(\Stn\cap \St)\backslash \So}|}\geq \s-\ell-\so$. 
Using \eqref{eq:pht_3} and the fact that each element of $\vt_{(\Stn\backslash \St)\cap \So}$ is selected from the largest $|{(\Stn\backslash \St)\cap \So}|$ elements of $-\eta\cdot \gt_{(\Stn\backslash \St)\cap \So}$  we have: 
\begin{equation}
  \label{eq:pht_4}
  \eta^2\|\gt_{\So\backslash \St}\|_2^2\leq  \eta^2\left(\|\gt_{(\Stn\backslash \St)\cap \So}\|_2^2+\|\gt_{\So\backslash (\Stn\cup \St)}\|_2^2\right)\leq \left(\eta^2\|\gt_{(\Stn\backslash \St)\cap \So}\|_2^2+\frac{\so}{\s-\ell-\so}\|\btt_{\St\backslash \So}\|_2^2\right).
\end{equation}
Using the above equation and Lemma~\ref{lem:full_1}, we have:
\begin{equation}
  \label{eq:pht_5}
  \frac{2}{\sc}\left(f(\btt)-f(\bto)\right)\leq \frac{1}{\sc^2}\|\gt_{(\Stn\backslash \St)\cap \So}\|_2^2+\left(\frac{1}{\sc^2\eta^2}\cdot \frac{\so}{\s-\ell-\so}-1\right)\|\btt_{\St\backslash \So}\|_2^2,
\end{equation}
Using $\s\geq 4\left(\frac{\ss}{\sc}\right)^2\so$, we have: 
\begin{equation}
  \label{eq:pht_7}
\|\gt_{\Stn\backslash \St}\|_2^2\geq 2\sc \left(f(\btt)-f(\bto)\right).
\end{equation}

{\bf Case 3}:  $|\Stn\backslash \St|\geq |\So\backslash \St|$. Now, as  $\vt_{\Stn\backslash \St}$ is obtained by selecting $|\Stn\backslash \St|$ largest elements of $\z^t_{\overline{\St}\cup bot^t}$. Hence, using Lemma~\ref{lem:full_1}, we have: 
\begin{equation}
  \label{eq:pht_1}
  \|\vt_{\Stn\backslash \St}\|_2^2\geq \|\z^t_{\So\backslash \St}\|_2^2=\eta^2\|\gt_{\So\backslash \St}\|_2^2\geq 2\eta^2\cdot \sc \cdot \left(f(\btt)-f(\bto)\right). 
\end{equation}
The lemma now follows by combining \eqref{eq:pht_0}, \eqref{eq:pht_9}, \eqref{eq:pht_7}, and \eqref{eq:pht_1}
%Now, similar to the analysis of \cite{JainTD11}, we analyse $\|\gt_{\Stn\backslash \St}\|_2^2$ using three mutually exclusive cases: 
%\begin{itemize}
%\item $|\Stn\backslash \St|< \ell$: 
%\item {\bf Case 2}: 
%\item {\bf Case 3}: 
%\end{itemize}
\end{proof}

%%% Local Variables: 
%%% mode: latex
%%% TeX-master: "rsc_iht"
%%% End: 

%\input{pht_proof.tex}
%\clearpage
\section{Supplementary Experimental Results}
Below we present plots that were not included in the main text.

\begin{figure*}[ht]
               \centering
			   \hspace*{-4ex}
               \subfigure[\!\!\!\!\!\!\!\!]{
                              \includegraphics[scale=0.55]{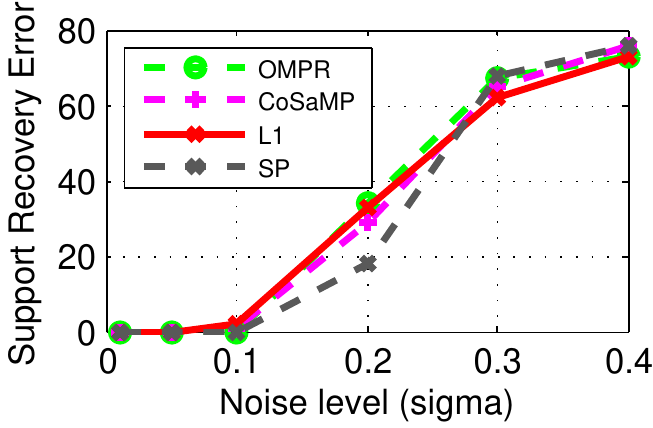}
                              \label{subfig:noise-supp}
               }\hspace*{-1ex}
               \subfigure[\!\!\!\!\!\!\!\!]{
                              \includegraphics[scale=0.55]{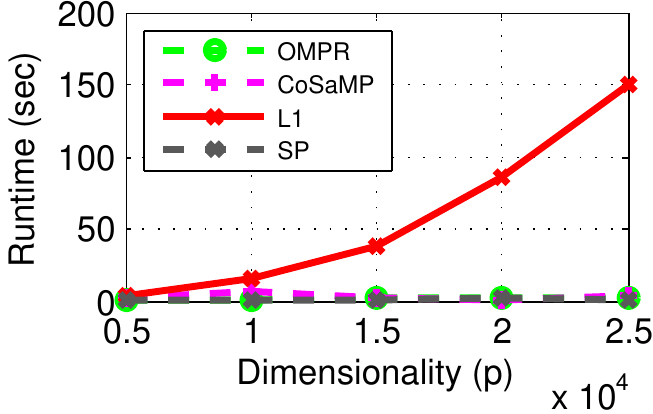}
                              \label{subfig:dimen-supp}
               }\hspace*{-1ex}           
               \subfigure[\!\!\!\!\!\!\!\!]{
               				
                              \includegraphics[scale=0.55]{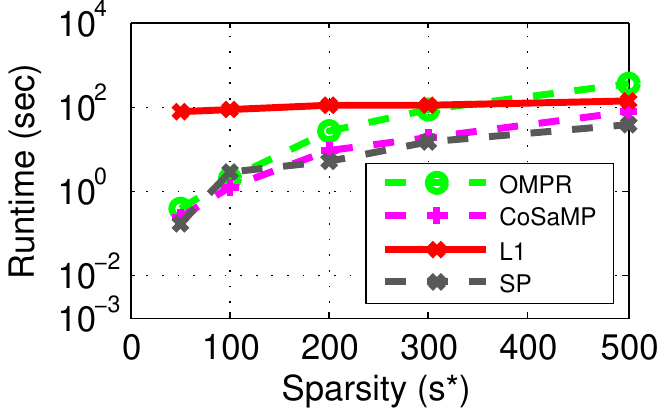}
                              \label{subfig:sparsity-supp}
               }
               %\vspace*{-2ex}
               \caption{Counterparts of Figure~\ref{fig:expts} for OMPR, CoSaMP and L1. 
               %\vspace*{-1ex}
				}
               \label{fig:expts-supp-1}
\end{figure*}

\begin{figure*}[ht]
               \centering
			   \hspace*{-4ex}
               \subfigure[\!\!\!\!\!\!\!\!]{
                              \includegraphics[scale=0.55]{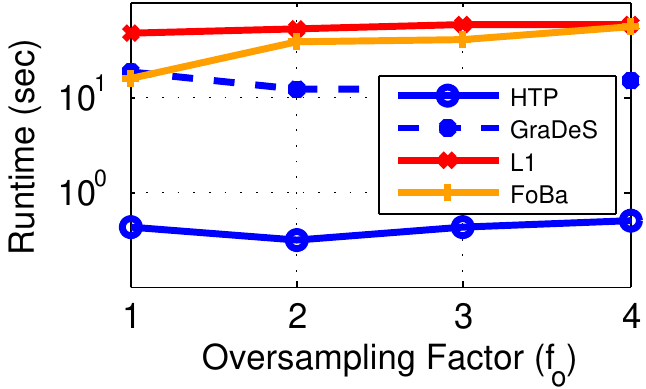}
                              \label{subfig:noise-supp-2}
               }\hspace*{-1ex}
               \subfigure[\!\!\!\!\!\!\!\!]{
                              \includegraphics[scale=0.55]{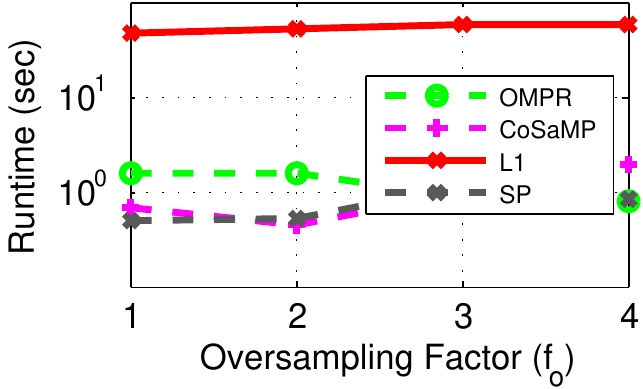}
                              \label{subfig:dimen-supp-2}
               }
               %\vspace*{-2ex}
               \caption{The effect of increasing sample sizes relative to the base value $\so\cdot\log p$ on runtime.
               %\vspace*{-1ex}
				}
               \label{fig:expts-supp-2}
\end{figure*}

%%% Local Variables: 
%%% mode: latex
%%% TeX-master: "lowsparse_matrix"
%%% End: 

\end{document}